\documentclass{article} 
\usepackage{hyperref}
\usepackage{url}

\usepackage{amssymb,dsfont,amsmath,mathtools,natbib}
\usepackage{amsthm} \usepackage{enumitem}
\usepackage{bbm}
\usepackage{algorithm}
\usepackage{algorithmic}

\usepackage
[
a4paper,
left=2.5cm,
right=2.5cm,
top=3cm,
bottom=4cm,
]
{geometry}

\newtheorem{theorem}{Theorem}
\newtheorem{proposition}{Proposition}
\newtheorem{lemma}{Lemma}

\usepackage{xspace}

\def\R{\mathbb{R}}
\def\1{\mathbbm{1}}

\def\a{\mathbf{a}}
\def\m{\mathbf{m}}
\def\b{\mathbf{b}}

\def\P{\mathbf{P}}
\def\p{\mathbf{p}}
\def\C{\mathbf{C}}
\def\D{\mathbf{D}}
\def\I{\mathbf{I}}

\def\mX{\mathcal{X}}

\newcommand{\Ex}{{\rm I\kern-.3em E}}

\newcommand{\argmin}{\mathop{\mathrm{arg\,min}}}

\newcommand{\eg}{{\em e.g.,~}}

\newcommand{\rev}[1]{\textcolor{black}{#1}}

\usepackage[colorinlistoftodos,bordercolor=orange,backgroundcolor=orange!20,linecolor=orange,textsize=scriptsize]{todonotes}

\newif\iflongversion
\longversionfalse
\newcommand{\longversion}[2]{\iflongversion #1 \else #2  \fi}

\begin{document}

\title{Optimal Transport for Conditional Domain Matching \\ and Label Shift }

\author{A. Rakotomamonjy\footnote{Criteo AI Lab, Paris, France, alain.rakoto@insa-rouen.fr}      \and
	R. Flamary\footnote{CMAP, Ecole Polytechnique} \and
	G.~Gasso\footnote{LITIS, INSA de Rouen Normandie} \and
	M. El Alaya\footnote{LMAC, Université Technologique de Compiègne} \and
	M. Berar\footnote{LITIS, Université de Rouen Normandie} \and
	N.~Courty\footnote{IRISA, Université de Bretagne Sud}
}

\maketitle

\begin{abstract}
We address the problem of unsupervised domain adaptation
under the setting of generalized target shift (joint class-conditional  and label shifts). For this framework, we theoretically show that, for good generalization, it is necessary to learn a latent representation in which both marginals and class-conditional distributions are aligned across domains. For this sake,
we propose a learning problem
that minimizes importance weighted loss in the source domain and a Wasserstein distance between weighted marginals. For a proper weighting, we provide an estimator of target label proportion by blending mixture  estimation and optimal matching by optimal transport. This estimation comes with theoretical guarantees of correctness under mild assumptions. 
Our experimental results show that our method performs better on average than competitors across a range domain adaptation problems including \emph{digits},\emph{VisDA} and \emph{Office}. Code for this paper is available at \url{https://github.com/arakotom/mars_domain_adaptation}.
\end{abstract}

\section{Introduction}

\label{sec:intro}

{

Unsupervised Domain Adaptation (UDA) is a machine learning subfield that aims at
addressing issues due to the discrepancy of train/test, also denoted as source/test, data distributions.  
There exists a large amount of literature addressing the UDA problem under different assumptions. One of the most studied setting is based on the 
covariate shift assumption (\rev{marginal distributions on source and target} $p_S(x) \neq p_T(x)$ and \rev{conditional distributions} $p_S(y|x) = p_T(y|x))$
for which methods perform importance weighting \cite{sugiyama2007covariate} or aim at aligning the marginal distributions in some learned feature space \cite{pan2010domain,long2015learning,pmlr-v37-ganin15}. Target shift, also denoted as label shift \citep{scholkopf2012} assumes that \rev{ for the class prior probability}, $p_S(y) \neq p_T(y)$ while, for \rev{the class-conditional distributions}, we have $p_S(x|y) = p_T(x|y)$.
For this problem, most works seek at estimating either the ratio $p_T(y)/p_S(y)$
or the label proportions  \citep{lipton2018detecting,azizzadenesheli2018regularized,shrikumar2020adapting,pmlr-v89-li19b,pmlr-v89-redko19a}.

However as most models now learn the latent representation space, in practical
situations we have both a label shift ($p_S(y) \neq p_T(y)$) and class-conditional probability shift ($p_S(z|y) \neq p_T(z|y)$, $z$ being
a vector in the latent space).
For this {more general} DA assumption, denoted as generalized target shift, fewer works have been proposed. \citet{zhang2013domain} 
have been \rev{among} the first authors that proposed a methodology for handling both shifts. They used a kernel embedding of distributions for estimating importance
weights and for transforming samples so as to match class-conditional  distributions. \citet{gong2016domain} follow similar idea by assuming that there exists a linear mapping that maps source class-conditionals to the target ones.
For addressing the same problem \citet{pmlr-v97-wu19f} introduced a so-called asymmetrically-relaxed distance on distributions that allows to
mitigate the effect of label shift when aligning marginal distributions.
Interestingly, they also show that, when marginals in the latent space are aligned, error in the target domain is lower-bounded
by the mismatch of label distributions between the two domains. Recently, \citet{combes2020domain} have presented a theoretical analysis of this problem
showing that target generalization can be achieved by matching label proportions 
and class-conditionals in both domains. The key component of their algorithm relies on a importance weight estimation of the label distributions. Unfortunately, although
relevant in practice, their label distribution estimator got theoretical guarantee only when class conditionals match across domains
 and empirically  breaks as soon as class conditionals mismatch becomes large enough. 

Our work addresses  UDA with generalized target shift and
we make the following  contributions. 
From a theoretical side, we introduce a bound which clarifies the role of the label shift and class-conditional shift in the target generalization error bound. 
Our theoretical analysis emphasizes
the importance of learning with same label distributions in source and target domains while seeking at minimizing class-conditional shifts in a latent space. Based on this theory, we derive a learning problem and an algorithm which aims
at minimizing Wasserstein distance between weighted marginals while ensuring
low empirical error in a weighted source domain. Since  a weighting scheme
requires the knowledge of the label distribution in the target domain, we solve this estimation problem by blending a consistent mixture proportion estimator and an optimal matching assignment problem. While conceptually simple, our strategy is  supported by  theoretical guarantees of correctness.
Then, given the estimated label proportion in the target domain, we theoretically
show that finding a latent space in which the Wasserstein distance \rev{between the weighted source marginal distribution and the target one is zero,} guarantees that class-conditionals are also matched. 
We illustrate in our
experimental analyses how  our
algorithm (named MARS from Match And Reweight Strategy)  copes with label and class-conditional shifts
and show that it performs better than other generalized target shift competitors on several UDA
 problems.  
}

\section{Notation and Background}
\label{sec:related}

Let $\mathcal{X}$ and $\mathcal{Y}$ be the input and output space.
We denote by $\mathcal{Z}$ the latent space and
 $\mathcal{G}$ the class of representation mappings from $\mathcal{X}$ to $\mathcal{Z}$. 
 Similarly, $\mathcal{H}$ represents the hypothesis space, which
 is a set of functions from $\mathcal{Z}$ to $\mathcal{Y}$. A labeling function $f$ is a function from $\mathcal{X}$ to $\mathcal{Y}$. Elements of $\mathcal{X}$, $\mathcal{Y}$ and $\mathcal{Z}$ are respectively noted
 as $x$, $y$ and $z$. For our UDA problem,
we assume
a learning problem with  source and target domains and respectively note as
$p_S(x,y)$ and $p_T(x,y)$  their joint distributions of features and labels.
We have at our disposal a labeled source dataset $\{(x_i^s,y_i^s)\}_{i=1}^{n_s}$ with $y_i^s \in \{1\dots C\}$ \rev{(or $\{0,1\}$ for binary classification)}
and only unlabeled examples from the target domain $\{x_i^t\}_{i=1}^{n_t}$ with all
$x_i \in \mathcal{X}$, sampled \emph{i.i.d} from their respective distributions.
We refer to the marginal distributions of the source and target domains in the latent space as
$p_S^g(z)$ and $p_T^g(z)$. Class-conditional
probabilities in the latent space and label proportion for class $j$ will be respectively noted as {$p_U^{j} \triangleq p_U(z|y=j)$ and $p_U^{y=j} \triangleq p_U(y=j)$} with $U \in \{S,T\}$. 
{Finally, we defer proofs of the theoretical results  to the appendix.}

\subsection{Domain Adaptation Framework}
Since the seminal work of \citet{pan2010domain,long2015learning,pmlr-v37-ganin15}, a common formulation of the covariate shift domain
adaptation problem is to learn a mapping of the source and target samples into
a latent representation space where the distance between their marginal
distributions is minimized and to learn a hypothesis that correctly predicts labels
of samples in the source domain. This typically translates into
the following optimization problem:
\begin{equation}
\label{eq:optimisation}
\min_{h,g} { \frac{1}{n} \sum_{i=1}^{n_s} L(y_i^s, h(g(x_i^s))}) 
+ {\lambda D(p_S^g, p_T^g)} + \Omega(h,g)
\end{equation}
where $h(\cdot)$ is the hypothesis, $g(\cdot)$ a representation mapping and  $L(\cdot,\cdot) : \mathcal{Y} \times \mathcal{Y} \mapsto \R^+$ is a continuous loss function differentiable
on its second parameter and $\Omega$
a regularization term. Here, $D(\cdot,\cdot)$ is a distance metric 
between distributions that measures discrepancy between source and 
target marginal distributions as mapped in a latent space  induced by $g$. {Most used distance
measures are MMD \cite{tzeng2014deep}, Wasserstein distance \cite{shen2018wasserstein} or Jensen-Shannon distance \cite{ganin2016domain}.

\subsection{Optimal Transport (OT)}

We provide here some background on optimal transport  
as it will be a key concept for assigning label proportion. More details can be found in \citet{peyre2019computational}.
Optimal transport measures the distance between two distributions over
a space $\mathcal{X}$ given a transportation cost $c: \mX \times \mX
\rightarrow \R^+$. It seeks for an optimal coupling between the two measures
that minimizes a transportation cost. In a discrete case, we denote the
two measures as $\mu = \sum_{i=1}^n a_i \delta_{x_i}$ and 
$\nu = \sum_{i=1}^m b_i \delta_{x_i^\prime}$. The 
Kantorovitch relaxation
of the OT problem seeks for a transportation coupling $\P$ that minimizes the problem 
\begin{equation}\label{eq:wd}
\min_{\P \in \Pi(\a,\b)} \langle \C, \P \rangle
\end{equation}
where $\C \in \R^{n \times m}$ is the matrix of all pairwise
costs, $\C_{i,j} = c(x_i,x_j^\prime)$ and 
$
\Pi(\a,\b) = \{\P \in \R_+^{n\times m}| \P \mathbf{1} = \a, \P^\top \mathbf{1} = \b \}
$ is the transport polytope between the two distributions. The above problem is known as the discrete optimal
transport problem and in the specific case where $n=m$
and the weights $\a$ and $\b$ are positive and uniform then the solution of
the above problem is a {scaled} permutation matrix \citep{peyre2019computational}.
One of the key features of OT that we are going to exploit for
solving the domain adaptation problem is its ability to find
correspondences between {samples} in an unsupervised way by exploiting the {underlying space geometry}. These features
have been for instance exploited for unsupervised word translation \cite{pmlr-v89-alvarez-melis19a,alaux19}.

\section{Theoretical Insights}
\label{sec:mars}

In this work, we are interested in a situation where both 
class-conditional  and label shifts occur between source and target distributions \emph{i.e} there exists some $j$ so that  $p_S(z|y=j) \neq  p_T(z|y=j)$ and $p_S^{y=j} \neq p_T^{y=j}$. Because we have these two sources of mismatch, the resulting  domain adaptation problem is difficult and aligning marginals is not sufficient \citet{pmlr-v97-wu19f}.

For better understanding  the key aspects of the problem,  we provide an upper bound on the
target generalization error which exhibits the role of class-conditional and label distribution mismatches.
For a sake of simplicity, we will consider binary classification problem.
Let $\mathcal{X}$ be the input space and assume that \rev{the function $f: \mathcal{X} \mapsto \{0,1\}$ is the domain-invariant labeling function}, which is
a classical assumption in DA \citep{pmlr-v97-wu19f,shen2018wasserstein}.
For a domain $U$, with $U=\{S,T\}$, the induced marginal probability of samples
in $\mathcal{Z}$ is formally defined  as $p_U^g(A)=p_U(g^{-1}(A))$ for any subset $A \subset \mathcal{Z}$ and $g^{-1}(A)$ being potentially a set ($p_U^g(A)$ is thus the push-forward of $p_U(x)$ by $g(\cdot)$). Similarly, we define the conditional distribution
$g_U(\cdot|z)$ such that $p_U(x) = \int g_U(x|z) p_u^g(z) dz$ holds for all $x \in \mathcal{X}$. For  a representation mapping $g$, an hypothesis $h$ and the labeling function $f$, the expected risk is defined as 
$\varepsilon_{U}(h \circ g,f) \triangleq \mathbb{E}_{x \sim p_U}[|h(g(x)) - f(x)|] 
=  \mathbb{E}_{z \sim p_U^z}[|h(z) - f_U^g(z)|] \triangleq \varepsilon_{U}^z(h,f_U^g) $
with $f_U^g$ being a domain-dependent labeling function defined as
$f_U^g(z)=\int f(x)g_U(x|z)dx$.

Now, we are in position to derive a bound on the target error but first, we introduce a key intermediate result.

\begin{lemma} \label{lemme:bound} Assume two marginal distributions $p_S^g$ and $p_T^g$, with $p_U^g = \sum_{k=1}^C p_U^{y=k} p_{U}^k$, $U=\{S,T\}$.
	For all $p_T^y$, $p_S^y$ and for any continuous class-conditional density distribution $p_S^k$ and $p_T^k$ such that for all $z$ and $k$, we have $p_S(z|y=k)>0$ and $p_S(y=k)>0$,
	the  inequality 
	$
	\sup_{k,z} [w(z) S_k(z)] \geq 1
	$ holds
	with $S_k(z) = \frac{p_T(z|y=k)}{p_S(z|y=k)}$ and $w(z) = \frac{p_T^{y=k}}{p_S^{y=k}}$, if $z$ is of class $k$.
\end{lemma} 
\longversion{
\begin{proof}
	Let first show that for any $k$ the ratio $\sup_z \frac{p_T^k}{p_S^k} \geq 1$. 
	Suppose that there does  not exist a $z$  such that $\frac{p_T^k}{p_S^k} \geq 1$. This means that : $\forall z\,\, p_T^k<p_S^k$. By integrating those positive and continuous functions on their domains lead to the contradiction that the integral of one of them is not equal to 1. Hence, there must exists a $z$ such that $\frac{p_T^k}{p_S^k} \geq 1$. Hence, we indeed have ratio $\sup_z \frac{p_T^k}{p_S^k} \geq 1$.
	
	Using a similar reasoning, we can show that $\max_k \frac{p_T^{y=k}}{p_S^{y=k}} \geq 1 $. For a sake of completeness,
	we provide it here. Assume that $\forall k,\,\,p_T^{y=k}<p_S^{y=k}$.
	We thus have $\sum_k p_T^{y=k}< \sum_k p_S^{y=k}$. Since noth sums should be equal to $1$ leads to a contradiction.
	
	By exploiting these two inequalities, we have :
	$$\sup_{k,z} [w(z) S_k(z)] = \sup_{k} \left [w(z) \sup_z S_k(z)\right]\geq \sup_{k} w(z)  \geq 1$$
	
\end{proof}
}{}
{Intuitively, this lemma says that the maximum ratio between class-conditionals weighted by label proportion ratio is lower-bounded by 1, and that potentially, this bound can be achieved when both
$ p_S^{y=k}=p_T^{y=k}$ and $p_{S}^k = p_{T}^k$.}
 Interestingly,    \citet{pmlr-v97-wu19f}'s results involve a similar term  $\sup_z \frac{p_T^g(z)}{p_S^g(z)}$ for defining their assymetrically-relaxed distribution distance. But we use a finer modeling that allows us to 
explicitly disentangle the role of the class-conditionals and label distribution ratio. In our case, owing to this inequality, we can bound  one of the key term that upper bounds the generalization error in the target domain.
\begin{theorem}\label{th:bound}
	 Under the assumption of Lemma \ref{lemme:bound}, and assuming that any function $h \in \mathcal{H}$ is $K$-Lipschitz and $g$ is a continuous function then
	 for every function  $h$ and $g$, we have
	\begin{align}
	\varepsilon_{T}(h \circ g,f) &\leq \varepsilon_{S}(h \circ g,f)  + 2K \cdot WD_1(p_S^g,p_T^g)
	\nonumber\\ & + \left[1 + \sup_{k,z} w(z)S_k(z))\right] \varepsilon_{S}(h^\star \circ g,f) 
		\nonumber\\ & + \varepsilon_{T}^z(f_S^g,f_T^g) 	\nonumber
	\end{align}
	where $S_k(z)$ and $w(z)$ are as defined in Lemma \ref{lemme:bound}, $h^\star = \argmin_{h \in \mathcal{H}} \varepsilon_{S}(h \circ g;f)$ 
	and $\varepsilon_{T}^z(f_S^g,f_T^g) = \mathbb{E}_{z \sim p_T^z}[|f_T^g(z) - f_S^g(z)|] $ and $WD_1$ defined through its dual form as 			\begin{equation} \nonumber
	WD_1(p_S^g,p_T^g) = \sup_{\|v\|_L  \leq 1} \mathbf{E}_{z \sim p_S^g}w(z)v(z) - 	\mathbf{E}_{z \sim p_T^g} v(z)  
	\end{equation}
	with $w(\cdot)=1$. 
\end{theorem}

\longversion{
\begin{proof}
At first, let us remind the following result due to \citet{shen2018wasserstein}.
Given two probability distributions $p_S^g$ and $p_T^g$, we have
$$
\varepsilon_S^z(h,h^\prime) - \varepsilon_T^z(h,h^\prime) \leq 2K \cdot W_1(p_S^g,p_T^g)
$$
for every hypothesis $h$,\,$h^\prime$ in $\mathcal{H}$. Then, we have the following bound for the target error
	\begin{align}
	\varepsilon_{T}(h \circ g,f) &\leq \varepsilon_{T}(h \circ g,h^\star) + 
	\varepsilon_{T}(h^\star \circ g,f) \label{eq:1}\\
	 &\leq \varepsilon_{T}(h \circ g,h^\star) + \varepsilon_{S}(h \circ g,h^\star) - \varepsilon_{S}(h \circ g,h^\star) +
	\varepsilon_{T}(h^\star \circ g,f) \label{eq:2}\\
	& \leq \varepsilon_{S}(h \circ g,h^\star) + 	\varepsilon_{T}(h^\star \circ g,f)  + 2K W_1(p_S^g,p_T^g) \label{eq:3}\\
	& = \varepsilon_{S}^z(h,h^\star) + 	\varepsilon_{T}^z(h^\star,f_T^g)  + 2K W_1(p_S^g,p_T^g) \label{eq:4} \\
	&  \leq \varepsilon_{S}^z(h,f_S^g) + \varepsilon_{S}^z(h^\star,f_S^g) +	\varepsilon_{T}^z(h^\star,f_T^g)  + 2K W_1(p_S^g,p_T^g) \label{eq:5} \\
	&  \leq \varepsilon_{S}^z(h,f_S^g) + \varepsilon_{S}^z(h^\star,f_S^g) +	\varepsilon_{T}^z(h^\star,f_S^g)+ \varepsilon_{T}^z(f_S^g,f_T^g) + 2K W_1(p_S^g,p_T^g)\label{eq:6}
	\end{align}
where the lines \eqref{eq:1}, \eqref{eq:5}, \eqref{eq:6} have been obtained using triangle inequality, Line \eqref{eq:3} by applying Shen's et al.
above inequality, Line \eqref{eq:4} by using $\varepsilon_{U}(h \circ g,f) =  \varepsilon_{U}^z(h,f_U^g)$.
Now, let us analyze the term $\varepsilon_{S}^z(h^\star,f_S^g) +	\varepsilon_{T}^z(h^\star,f_S^g)$. Denote as $r_S(z) = |h^\star(z)-f_S^g(z)|$. Hence, we have
\begin{align}
	\varepsilon_{S}^z(h^\star,f_S^g) +	\varepsilon_{T}^z(h^\star,f_S^g) &= \int r_S(z) [p_S^g(z) + p_T^g(z)] dz \\
	&= \sum_k p_S(y=k) \int r_S(z)p_S^g(z|y=k)\big[1 + \frac{p_T(y=k)}{p_S(y=k)}S_k(z)\big] dz \label{eq:10}\\
	& \leq \left(1 + \max_{k,z} [w(z) S_k(z)] \right) 	\varepsilon_{S}(h^\star,f_S^g)
\end{align}
where Line \eqref{eq:10} has been obtained by expanding marginal distributions. Merging the last inequality into Equation \eqref{eq:6} concludes the proof.
\end{proof}
}{}

Let us analyze the terms that bound the target generalization error. The first term $\varepsilon_{S}(h \circ g,f) \triangleq \varepsilon_{S}^z(h,f_S^g)$ can be understood as the error induced by the hypothesis $h$ and the mapping $g$. This term is controllable through an empirical risk minimization approach as 
we have some supervised training data available from the source domain. The second term is the Wasserstein distance between the marginals of the source and target distribution in the latent space. Again, this can be minimized based on empirical examples and the Lipschitz constant $K$ can be controlled either by regularizing the model $g(\cdot)$ or by properly setting the architecture of the neural network model used for $g(\cdot)$. The last term $\varepsilon_{T}(f_S^g,f_T^g)$ is not directly controllable \citep{pmlr-v97-wu19f} but it becomes
zero if the latent space labelling function is domain-invariant which is a reasonable assumption especially when latent joint distributions of the source and target domains are equal. The remaining term that we have to analyze is 
$\sup_{k,z} [w(z) S_k(z)]$ which according to Lemma \ref{lemme:bound} is lower bounded by 1. This lower bound is attained when the label distributions are equal and class-conditional distributions are all equal and
in this case, the joint distributions in the source and target domains are equal and thus  $\varepsilon_{T}^z(f_S^g,f_T^g)=0$.

\section{Match and Reweight Strategy}

\subsection{The  Learning Problem}
The bound in Theorem \ref{th:bound} suggests that a good model
should: i) look for a latent representation mapping $g$ and a hypothesis $h$ that generalizes
well on the source domain, ii) have minimal Wasserstein distance between marginal distributions
of the latent representations while having class-conditional probabilities that match, and iii)  learn from source data with equal label proportions as the target so as to have $w(z)=1$ for all $z$.
For yielding our learning problem, we will translate these properties into an optimization problem.

{At first, let us note that one simple and efficient way to handle mismatch in label distribution is to consider importance weigthing in the source domain. Hence, instead of learning from the  marginal source distribution $p_S = \sum_{k=1}^C p_S^{y=k} p_{S}^k$, we learn from a reweighted version denoted as $p_{\tilde S} =  \sum_{k=1}^C p_T^{y=k} p_S^k$, as proposed by \citet{sugiyama2007covariate,combes2020domain}, so that no label shift occurs between $p_{\tilde S}$ and $p_T$ .
 This approach  needs an estimation of $p_T^{y=k}$ that we will detail in the next subsection, but interestingly, in this case, for Theorem \ref{th:bound}, we will  have $w(z)=
\frac{p_T^{y=k}}{p_{\tilde S}^{y=k}} = \frac{p_T^{y=k}}{ p_T^{y=k}}=1
$.
Then, based on the bound in Theorem \ref{th:bound} applied to $p_{\tilde S}$ and $p_T$,  we propose to learn
 the functions 	$h$ and $g$ by solving the problem
	\begin{equation}
	\label{eq:optimisation2} \small
	\min_{g,h} \frac{1}{n} \sum_{i=1}^{n_s}  w^\dagger(x_i^s) L(y_i^s, h(g(x_i^s))) 
	+  \lambda WD_1(p_{\tilde S}^g, p_T^g) + \Omega(h,g)
	\end{equation}
	where  the importance weight $w^\dagger(x_i^s)=\frac{p_T^{y=y_i}}{ p_S^{y=y_i}}$ allows to simulate sampling from ${p}_{\tilde S}^g$ given
	 ${p}_S^g$, and the discrepancy  between marginals is the  Wasserstein distance
		\begin{equation} 	\label{eq:wddual}
	WD_1(\tilde p_s^g,p_t^g) = \sup_{\|v\|_L  \leq 1} \mathbf{E}_{z \sim p_S^g}w^\dagger(z)v(z) - 	\mathbf{E}_{z \sim p_T^g} v(z).  
	\end{equation}
		 The first term of equation \eqref{eq:optimisation2} corresponds to the empirical loss related to the error $\varepsilon_{\tilde S}$ in Theorem
\ref{th:bound} while the distribution divergence aims at minimizing distance between marginal probabilities, the second term in that theorem. In the next subsections, we will make clear why the Wasserstein distance is used as the divergence and provide conditions and guarantees for having $ WD_1(\tilde p_S^g, p_T^g)=0 \implies WD(p_{S}^k,p_{T}^k)=0$, i.e. perfect class-conditionals matching, and thus $S_k(z) = 1$ for all $k,z$. Recall that in this case, the lower bound on $\max_{k,z} [w(z) S_k(z)]$ will be attained.

Algorithmically, for solving the problem in Equation \eqref{eq:optimisation2}, we employ a classical adversarial learning strategy. It is based on a standard back-propagation strategy using stochastic gradient descent  (detailed in Algorithm \ref{alg:mars}). We estimate the label proportion using Algorithm \ref{alg:pwl} and then use this proportion for computing the importance weights $w(\cdot)$. The first part of the algorithm consists then in computing the weighted Wassertein distance using gradient penalty \citep{gulrajani2017improved}. Once this distance is computed,
we back-propagate the error through the parameters of the feature extractor 
$g$ and the classifier $f$. In practice, we use weight decay as regularizer $\Omega$ over
the representation mapping and classifier functions $g$ and $h$.

\begin{algorithm}[t]
	\caption{Training the full MARS model }
	\label{alg:mars}
	\begin{algorithmic}[1]
		\REQUIRE{ $\{(x_i^s,y_i^s)\}, \{x_i^t\}$, number of classes $C$,  batch size $B$, number of critic iterations $n$}
				\STATE Initialize representation mapping $g$, the classifier $h$ and the domain critic $v(\cdot)$, with parameters $\theta_h$, $\theta_g$, $\theta_v$
		\REPEAT	
		\STATE estimate $\p_T$ from $\{x_i^t\}$ using Algorithm \ref{alg:pwl}   \hfill\COMMENT{done every $10$ iterations}
		\STATE sample minibatches $\{(x_B^s,y_B^s)\}, \{x_B^t\}$ from $\{(x_i^s,y_i^s)\}$ and $\{x_i^t\}$
		
		\STATE  compute $\{w_i^\dagger\}_{i=1}^C$ given $\{(x_B^s,y_B^s)\}$ and $\p_T$
		\FOR {$t=1,\cdots, n$} 
		\STATE $z^s \leftarrow g(x_B^s)$,\,\,$z^t \leftarrow g(x_B^t)$
				\STATE compute gradient penalty $\mathcal{L}_{\text{grad}}$ 				\STATE compute empirical Wasserstein dual loss $\mathcal{L}_{wd} = \sum_i w^\dagger(z_i^s)v(z_i^s)  -  \frac{1}{B}\sum_i v(z_i^t) $ 
		\STATE $\theta_v \leftarrow \theta_v + \alpha_v \nabla_{\theta_v} [\mathcal{L}_{wd} - \mathcal{L}_{\text{grad}}]$
		\ENDFOR
		\STATE compute the weighted classification loss $\mathcal{L}_w = \sum_i w^\dagger(z_i^s) L(y_i^s,h(g(x_i^s)))$
		\STATE $\theta_h \leftarrow \theta_h - \alpha_h \nabla_{\theta_h} \mathcal{L}_w  $
		\STATE $\theta_g \leftarrow \theta_g - \alpha_g \nabla_{\theta_g} [\mathcal{L}_{w} + {\lambda}  \mathcal{L}_{wd}]$
		 		\UNTIL{a convergence condition is met} 
	\end{algorithmic}
\end{algorithm}

\subsection{Estimating Target Label Proportion Using Optimal Assignment}

The above learning problem needs an estimation of $P_T(y)$ for weighting the classification loss and for computing the Wasserstein distance between $ p_{\tilde S}^g$ and $p_T^g$. Several approaches exist for estimating $p_T^y$ when class-conditional distributions in source and target matches \citet{pmlr-v89-redko19a,combes2020domain}. However, this is not the case in our general setting.  Hence, in order to make the problem tractable, we
will introduce some assumptions on the structure and geometry of the  class-conditional distributions in the target domain that allow us to provide guarantee on the correct estimation of $p_T^y$.

For achieving this goal, we first consider the target marginal distribution
as a mixture of models and estimate the  proportions
of the mixture. Next we aim at finding a permutation $\sigma(\cdot)$ that guarantees, under  mild assumptions, correspondence between the class-conditional probabilities of same class in the source and target domain. Then, this permutation allows us to correctly assign a class to each mixture proportion leading to a proper estimation of each class label proportion in the target domain.

In practice, for the first step, we assume that the target distribution is a mixture model with $C$  components $\{p_T^j\}$ and we want to estimate
the mixture proportion of each component.   
For this purpose, we have considered two alternative strategies coming from the literature :
	i) applying agglomerative clustering on the target samples tells us about the membership class of each sample and thus, the resulting clustering provides the proportion of each component in the mixture.
	ii) learning a Gaussian mixture model over the data in the target domain. This gives us both the estimate components $\{p_T^j\}$ and the proportion of  the mixture $\p_u$. Under some conditions on its initialization and assuming the model is well-calibrated, \citet{zhao2020statistical} have shown that the sample estimator asymptotically converges towards the true mixture model.

	\begin{algorithm}[t]
		\caption{Label proportion estimation in the target domain }
		\label{alg:pwl}
		\begin{algorithmic}[1]
			\REQUIRE{ $\{(x_i^s,y_i^s)\}, \{x_i^t\}$, number of classes $C$}
			\ENSURE{$\p_T$ : Estimated label proportion}
			\STATE $\{p_T^j\},\p_u\leftarrow$ Estimate a mixture with $C$ modes and related proportions from $\{x_i^t\}$.

			\STATE $\D\leftarrow$ Given $\mathcal{D}$, compute the matrix pairwise distance $\{p_S^i\}$ and  $\{p_T^j\}$ modes. 						\STATE $\P^\star\leftarrow $ Solve OT problem \eqref{eq:wd} with $\D$ and uniform marginals as in Proposition \ref{prop:dist}.
			\STATE $\p_T \leftarrow C \cdot \P^\star \p_u\quad$ Permute the mixture proportion on source ( $C\cdot\P^\star$ is a permutation matrix)
		\end{algorithmic}
	\end{algorithm}

\paragraph{Matching Class-conditionals With OT}
Since, we do not know to which class each component of the mixture in target domain
is related to, we assume that the conditional distribution in the source
and target domain of the same class can be matched owing to optimal assignment. The resulting permutation  would then help us
assign each label proportion estimated as above to the correct class-conditional. Figure \ref{fig:proposition1} in the appendix illustrates this matching problem.

{Let us  suppose that we have an estimation of all $C$ class-conditional
probabilities on source and target domain (based on  empirical distributions).
We want to solve an optimal assignment problem with respect to the class-conditional probabilities $\{p_S^i\}_{i=1}^C$ and $\{p_T^{j}\}_{j=1}^C$ 
and we clarify under which conditions on distance between class-conditional probabilities, the assignment problem solution achieves a correct matching of classes (\emph{i.e} $
p_S^i$ is correctly assigned to $p_T^{i}$ for all $i$).
Formally,
denote as $\mathbb{P}$ the set of probability distributions over $\R^d$ and assume
a metric over $\mathbb{P}$.
We want to optimally assign a finite number $C$ of probability distributions {of  $\mathbb{P}$} to another set of finite number $C$  of probability distributions {belonging to  $\mathbb{P}$}, in a minimizing distance sense. Based on a  distance  $\mathcal{D}$
between couple of class-conditional probability distributions, the assignment problems  looks for the permutation that solves 
$
\min_\sigma \frac{1}{C}\sum_j \mathcal{D}(p_S^j,p_T^{\sigma(j)}).
$ 
Note that the best permutation $\sigma^\star$ solution to this problem can be
retrieved by solving  a Kantorovitch relaxed version of the optimal transport \citep{peyre2019computational} 
with marginals $\a=\b= \frac{1}{C}\1$. Hence,
this OT-based formulation of the matching problem can be interpreted as
an optimal transport one between discrete measures of probability distributions of the form $\frac{1}{C}\sum_{j=1}^{C} \delta_{p_U^j}$. 
In order to be able to correctly match class-conditional probabilities in source and target domain by optimal assignement,
we ask ourselves:

\textit{Under which conditions the retrieved permutation matrix would correctly match the class-conditionals?
}

In other word, we are looking for conditions of identifiability of classes in the target domain based on their geometry with respect to the classes in source
domain. Our proposition below presents an abstract sufficient condition for
identifiability based on the notion of cyclical monotonicity and then we exhibit
some practical situations in which this property holds.

\begin{proposition}\label{prop:dist}
	Denote as $ \nu = \frac{1}{C}\sum_{j=1}^{C}\delta_{p_S^j}$ and
	$ \mu = \frac{1}{C}\sum_{j=1}^{C}\delta_{p_T^j}$, representing
	respectively the balanced weighted sum of class-conditionals probabilities in source and target domains. 
	Given $\mathcal{D}$ a distance over probability distributions, 
	assume that for any permutation $\sigma$ of $C$ elements, the following assumption,
	known as the $\mathcal{D}$-cyclical monotonicity relation,
	 holds	$$\sum_j \mathcal{D}(p_S^j,p_T^j)\leq \sum_j \mathcal{D}(p_S^j,p_T^{\sigma(j)})$$
		then solving the optimal transport problem between
	 $\nu$ and $\mu$ as defined in equation \eqref{eq:wd}  using  $\mathcal{D}$ as the ground cost matches correctly class-conditional probabilities.
\end{proposition}
\longversion{
\begin{proof}
	The solution $\P^*$ of the OT problem lies on an extremal point of $\Pi_C$. Birkhoff's theorem
	\cite{birkhoff:1946} states that the set of extremal points of $\Pi_C$ is the set of permutation 
	matrices so that there exists an optimal solution of the form $\sigma^* :  [1,\cdots,C] \rightarrow [1,\cdots,C]$. 
	The support of $\P^*$ is $\mathcal{D}$-cyclically monotone~\citep{ambrosio2013user,santambrogio2015optimal} (Theorem 1.38), meaning that 
	$ \sum_j^C \mathcal{D} (p_S^j,p_T^{\sigma^*(j)}) \leq  \sum_j^C \mathcal{D} (p_S^j,p_T^{\sigma(j)}), \forall \sigma \neq \sigma^*.$
							Then, by hypothesis, $\sigma^*$ can be identified to the identity permutation, and solving the optimal assignment problem 
	matches correctly class-conditional probabilities.
		\end{proof}
}{}
	While the cyclical monotonicity assumption above can be hard to grasp, there exists several situations where it applies. One condition that is simple and intuitive is when class-conditionals of same source and target  classes are "near" each other in the latent space. More formally, if we assume that
	$
    \forall j\,\, \mathcal{D}(p_S^j,p_T^j) \leq  \mathcal{D}(p_S^j,p_T^k)  \quad \forall\,k
	$, then summing over all possible $j$, 
	and choosing $k$ so that all the couples of $(j,k)$ form a permutation, we recover the cyclical monotonicity condition
	$\sum_j^C \mathcal{D} (p_S^j,p_T^{j}) \leq  \sum_j^ \mathcal{D} (p_S^j,p_T^{\sigma(j)}), \forall \sigma$. 
	Another more general condition on the identifiability of the target class-conditional can be retrieved by exploiting the fact that, for discrete optimal transport with uniform marginals, the support of optimal transport plan satisfies the cyclical monotonicity condition \citep{santambrogio2015optimal}. 
			 This is for instance the case, when $p_S^j$ and
	$p_T^j$ are Gaussian distributions of same covariance matrices and 
	the mean $m_T^j$ of each $p_T^j$ is obtained as a linear symmetric positive definite mapping of the mean $m_S^j$ of  $p_S^j$ and the distance $\mathcal{D}(p_S^j,p_T^{j})$ is
	$\|m_S^j - m_T^j\|_2^2$ 	\citep{courty2016optimal}. This situation would correspond to a linear shift of the class-conditionals of the source domain to get the target ones.
 	Figure \ref{fig:proposition1} illustrates  how our class-conditional matching algorithm performs
 	on a simple toy problem.  	 While our assumptions can be considered as strong, we illustrate in Figure \ref{fig:embeddings}, that the above hypotheses hold
 	for the VisDA problem, and lead afterwards to a correct matching of the class-conditionals.

	\begin{figure*}[t]
		\begin{center}
			~\hfill
			\includegraphics[width=5.5cm]{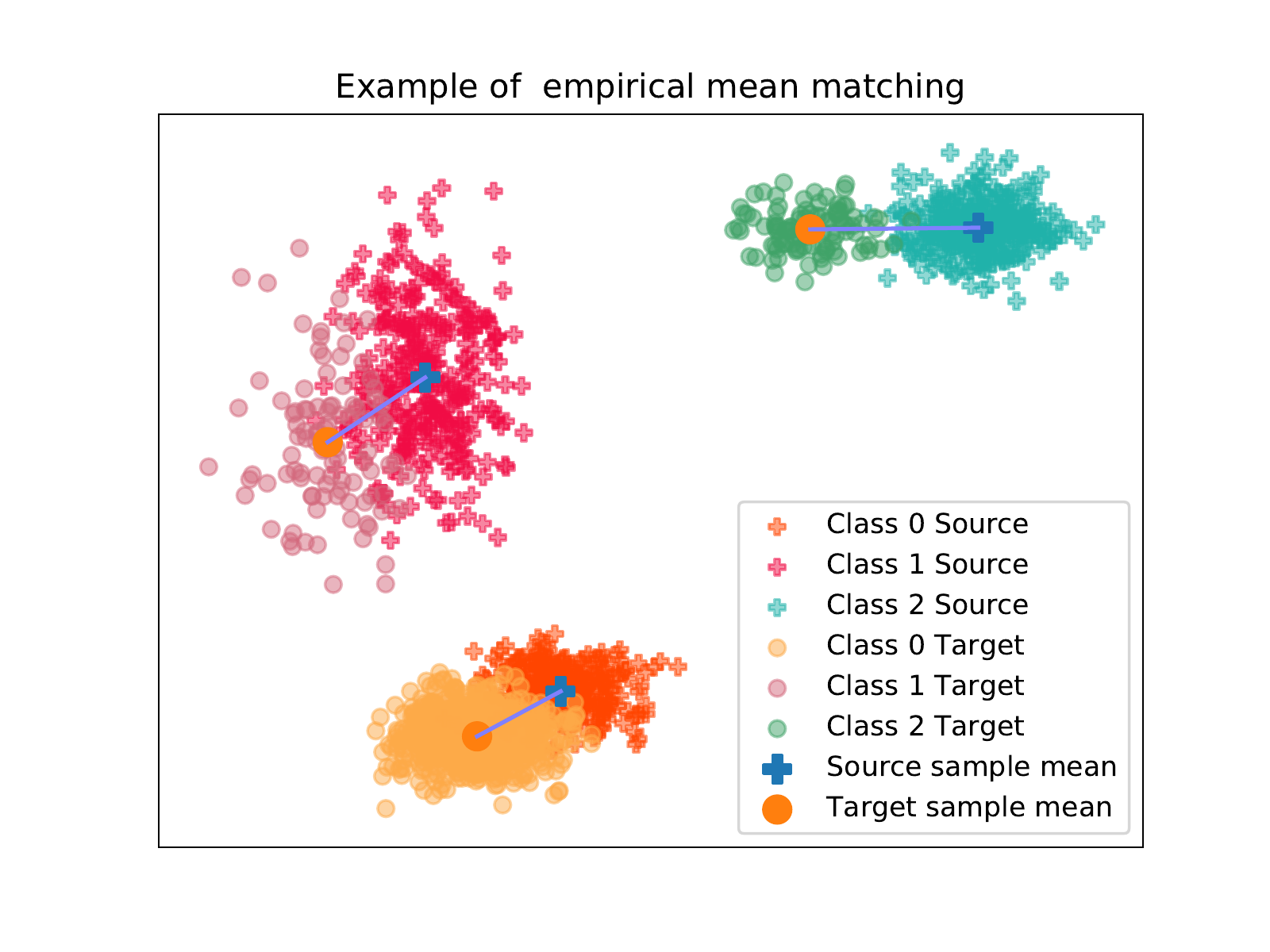}
			~\hfill~\includegraphics[width=5.5cm]{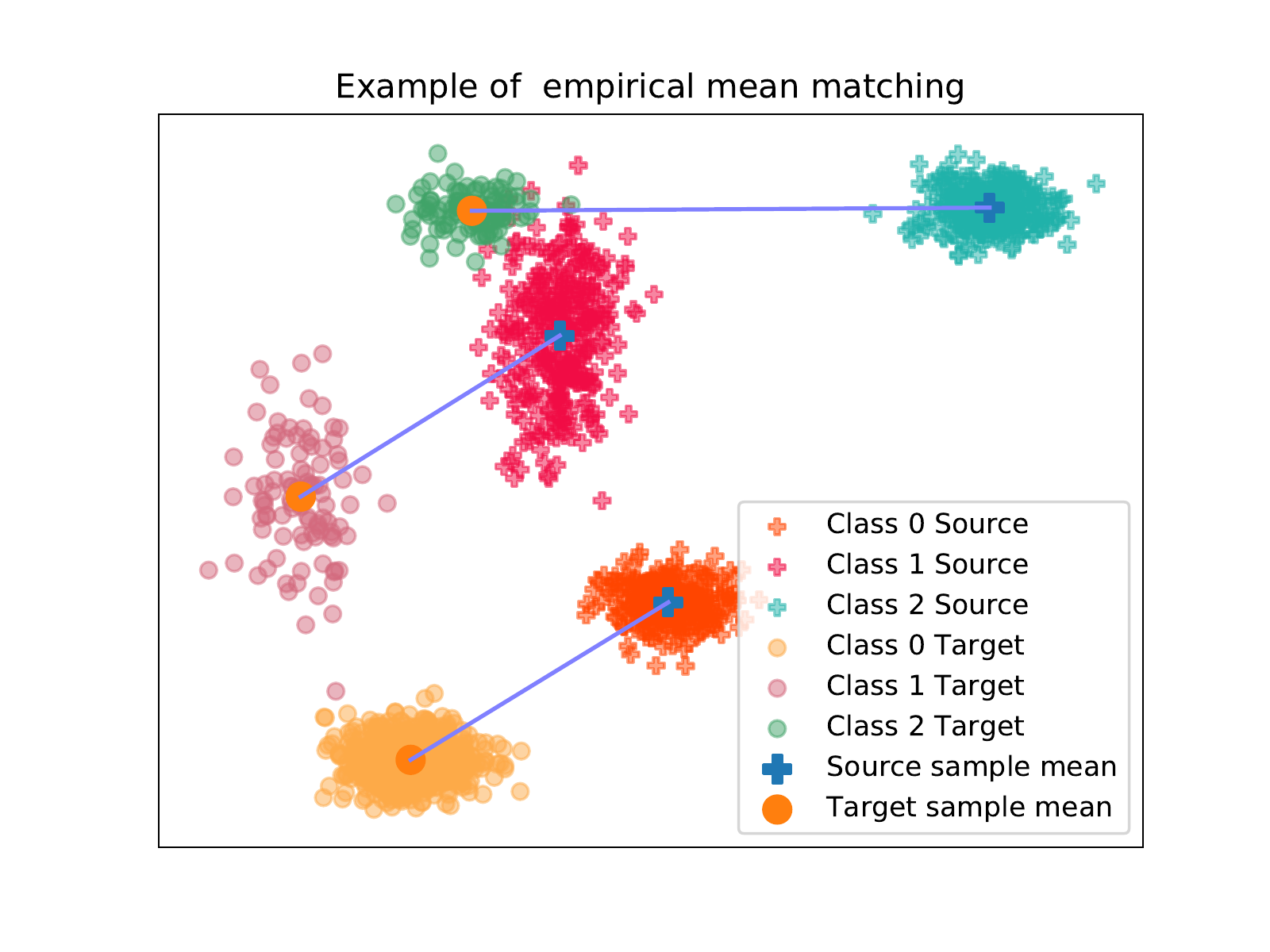}
			~\hfill~\\
			~\hfill\includegraphics[width=5.5cm]{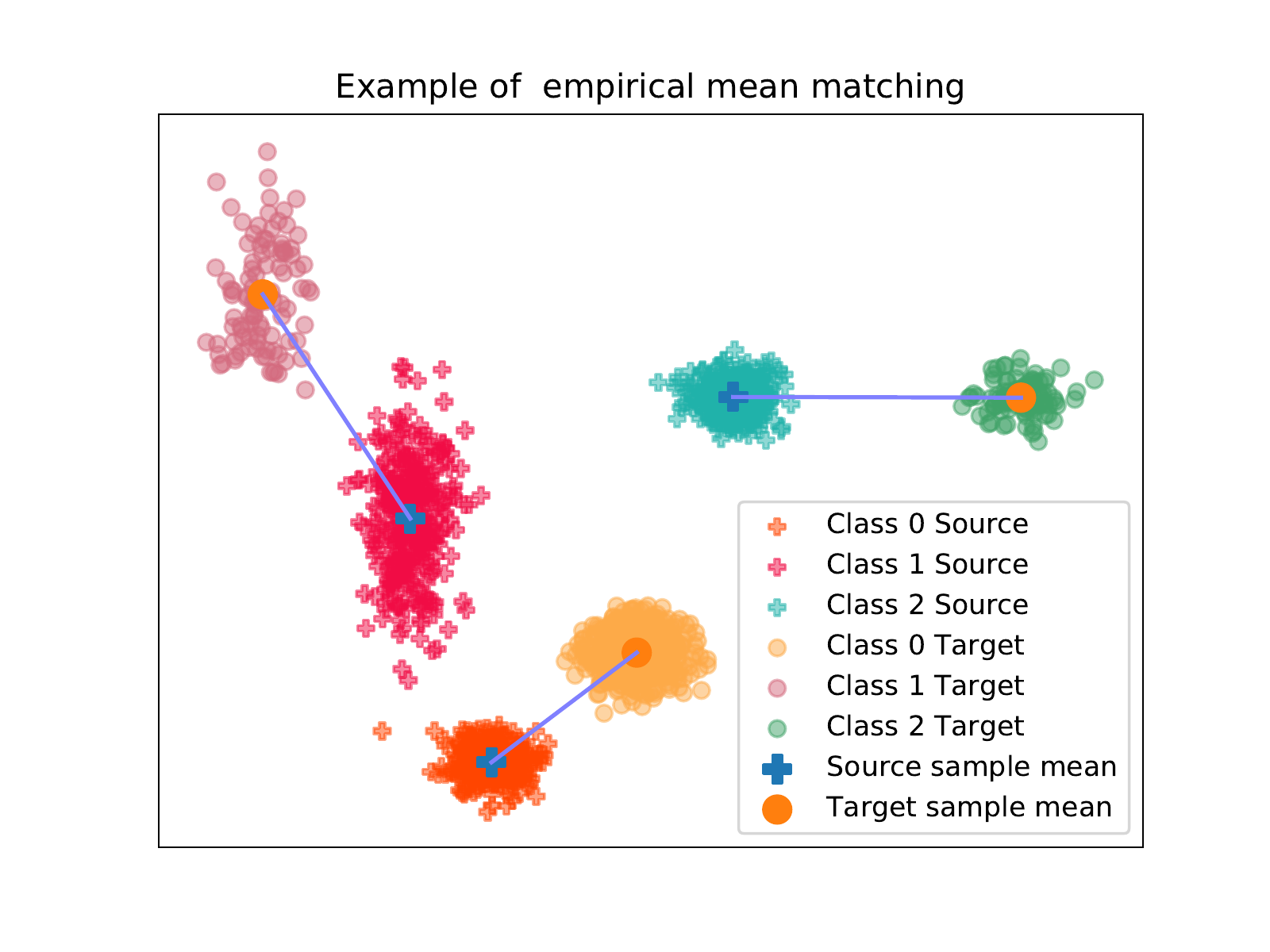}
			~\hfill~\includegraphics[width=5.5cm]{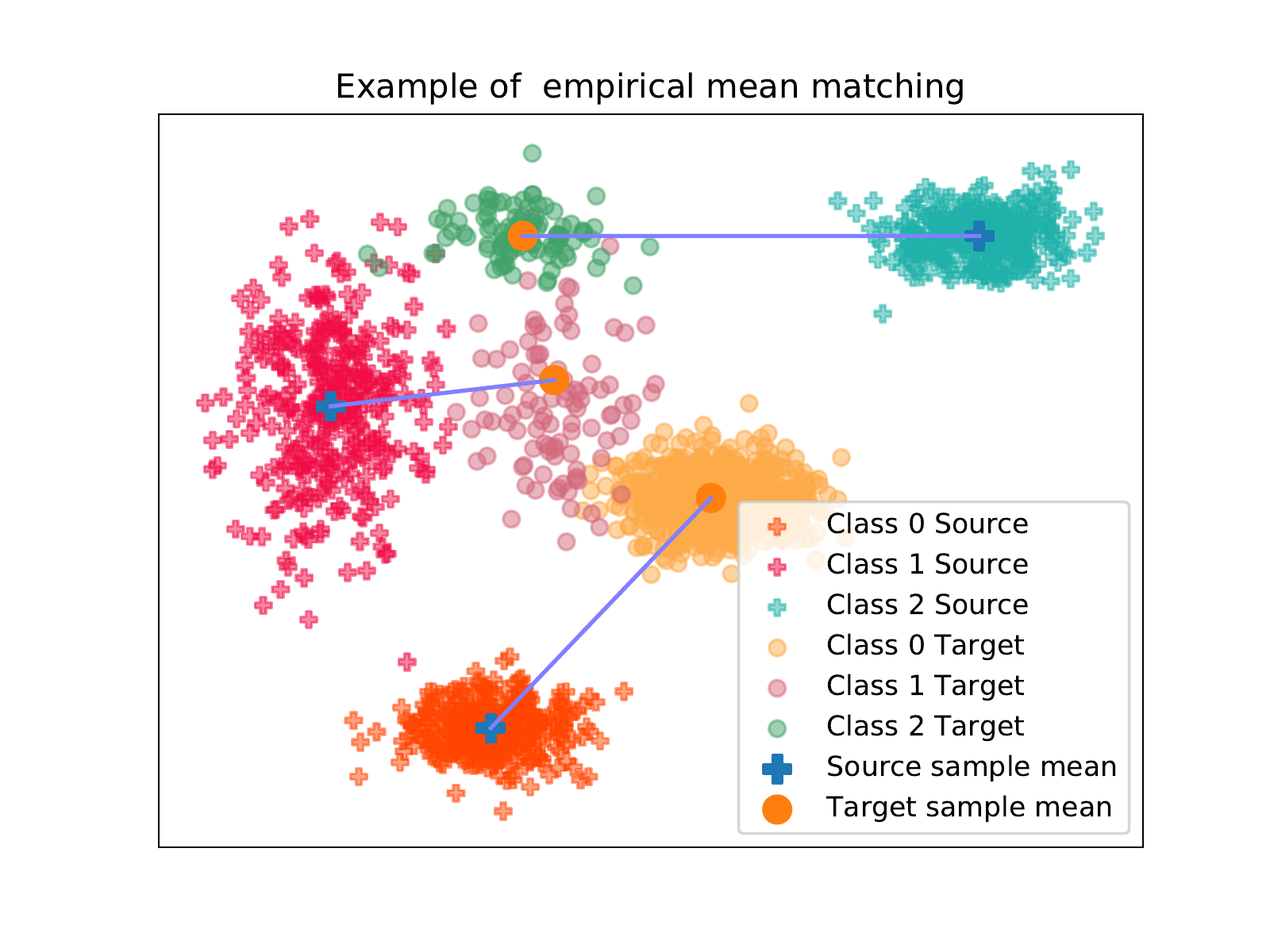}~\hfill~\\
			~\hfill\includegraphics[width=5.5cm]{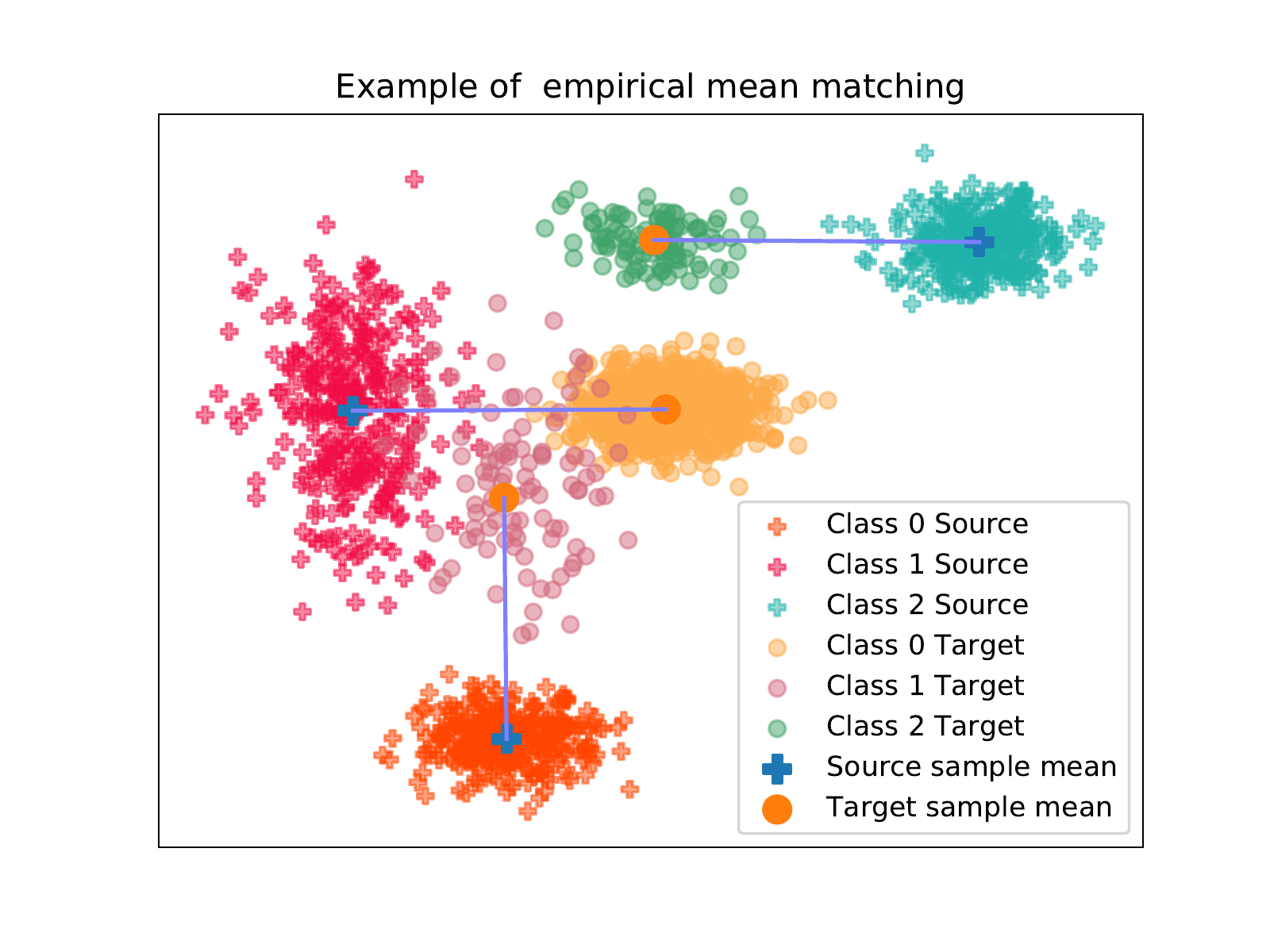}
			~\hfill~\includegraphics[width=5.5cm]{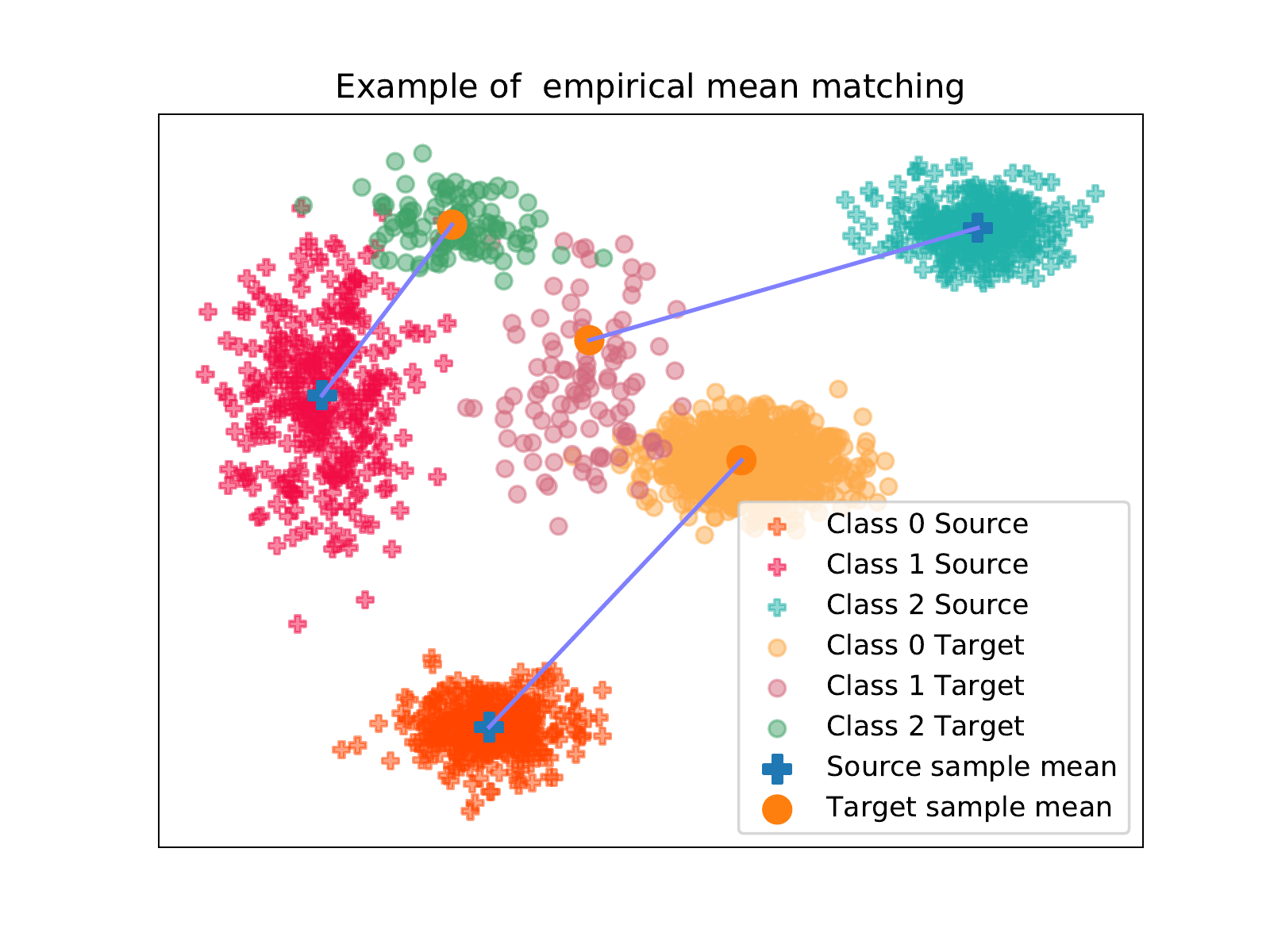}~\hfill~\\
		\end{center}
		\caption{Example of geometrical arrangments of the source and target class-conditional distributions that allows correct and incorrect matching of classes by optimal transport
			of empirical means (assuming correct estimation of these means). Blue lines denote the matching.
			(top-left) In this setting, the displacements of each class-conditionals is so that 
			for each class $i$ $\|\m_S^i - \m_T^i\|_2 \leq \|\m_S^i - \m_T^j\|_2$, for all $j$. We are thus in the first example that we gave as satisfying Proposition $1$. (top-right) Class-conditionals
			have been displaced such that the ``nearness'' hypothesis is not respected anymore. However, target class-conditional distributions are obtained by a linear Monge map of their source counterparts. This ensures that optimal transport allows their matchings
			(based on their means). (middle) We have illustrated two other examples of distribution arrangments  that allow class matching. (bottom) Two examples that break our assumption. In both cases, one target class-conditional is ``near''  another source class, without the global displacements of all target class-conditionals being uniform in direction. \label{fig:proposition1}}
	\end{figure*}

It is interesting to compare our assumptions on identifiability to other hypotheses proposed in the literature for solving (generalized) target shift problems. When handling only target shift, one
common hypothesis \cite{pmlr-v89-redko19a} is that class-conditional
probabilities are equal. This in our case boils down to have a $0$ distance
between $\mathcal{D}(P_S^j,P_T^j)$ guaranteeing matching under our more general assumptions. When both shifts
occur on 
labels and class-conditionals, \cite{pmlr-v97-wu19f} assume that there exists
continuity of support between the $p({z|y})$ in source and target domains. Again, 
this assumption may be related to the above minimum distance hypothesis if class-conditionals in source domain are far enough. 
Interestingly, one of the hypothesis of \citet{zhang2013domain} for handling
generalized target shift is that there exists a linear transformation
between the class-conditional probabilities in source and target domains.
This is a particular case of our Proposition \ref{prop:dist} and subsequent discussion, where the  mapping between class-conditionals  is supposed to be linear. Our conditions for correct matching and thus for identifying classes in the target domain are more general than those proposed
in the current literature.

\subsection{When Matching Marginals Lead To Matched Class-conditionals?}

In our learning problem, since one term we aim at minimizing is $WD_1({p}_{\tilde S}^g, p_T^g)$, with ${p}_{\tilde S}^g = \sum_j p_T^{y=j} p_S^j$ and $ p_T^g  = \sum_j p_T^{y=j} p_T^j$,
 we want to understand under which assumptions $WD_1({p}_{\tilde S}^g, p_T^g)=0$ implies that 
$p_{S}(z|y=j)=p_T(z|y=j)$ for all $j$, which is key for a good generalization as stated in Theorem \ref{th:bound}.
Interestingly, the assumptions needed for guaranteeing this implication
are the same as those in Proposition \ref{prop:dist}.

\begin{proposition}\label{prop:prop2}
	Denote as $\gamma$ 	the optimal coupling plan for distributions $\nu$ and $\mu$  defined as balanced weighted sum of class-conditionals 
		that is  $ \nu = \frac{1}{C}\sum_{j=1}^{C}\delta_{p_S^j}$ and
	$ \mu = \frac{1}{C}\sum_{j=1}^{C}\delta_{p_T^j}$
	 under assumptions given in  Proposition \ref{prop:dist}. 	Assume that the classes are ordered so that we have $\gamma= \frac{1}{C}
	\text{diag}(\1)$. Then  $\gamma'=\text{diag}(\a)$ is also optimal for the transportation problem
	with  marginals $ \nu^\prime = \sum_{j=1}^{C} a_j \delta_{p_S^j}$ and
	$ \mu^\prime = \sum_{j=1}^{C} a_j \delta_{p_T^j}$, with $a_j > 0, \forall j$. In addition,
	if the Wasserstein distance between $\nu^\prime$ and $\mu^\prime$ is $0$, it implies
	that the distance between class-conditionals are all $0$.
	\end{proposition}
\longversion{
	\begin{proof}
		By assumption and without loss of generality, the
	class-conditionals are arranged so that  $\gamma = \frac{1}{C}
	\text{diag}(\1)$.
	Because the weights in the marginals are
	not uniform anymore, $\gamma$  is not a feasible solution for the OT problem with 
	$\nu^\prime$ and  $\mu^\prime$
	but $\gamma^\prime = \text{diag}(\a)$ is. 
	Let us now show that any feasible non-diagonal plan $\Gamma$  has higher cost than
	$\gamma^\prime$ and thus is not optimal. At first, consider any permutation $\sigma$ of $C$ elements and its corresponding permutation matrix $\P_\sigma$, because $\gamma = \frac{1}{C}
	\text{diag}(\1)$ is optimal, the cyclical monotonicity relation $\sum_i \mathcal{D}_{i,i}\leq \sum_i \mathcal{D}_{i,\sigma(i)}$  holds true $\forall \sigma$. 
	Starting from $\gamma^\prime = \text{diag}(\a)$, any direction $\Delta_\sigma= -\I+\P_\sigma$ is a feasible direction (it does not violate the marginal constraints) and due to the  cyclical monotonicity, any move in this direction will increase the cost. Since any non-diagonal $\gamma_ z\in\Pi(\a,\a)$ can be reached with a sum of displacements $\Delta_\sigma$ (property of the Birkhoff polytope) it means that the transport cost induced by  $\gamma_z$ will always be greater or equal to the cost for the diagonal $\gamma^\prime$ implying that $\gamma^\prime$ is the solution of the OT problem with marginals $\a$.\\
	As a corollary, it is straightforward to show that 
	$
	W(\nu^\prime, \mu^\prime) = \sum_{i=1}^C \mathcal{D}_{i,i} a_i = 0 \implies \mathcal{D}_{i,i} = 0
	$
	as $a_i >0$ by hypothesis.
	\end{proof}
}{}

Applying this proposition with $a_j= p_T^{y=j}$ brings us the guarantee that under some geometrical assumptions on the class-conditionals in the latent space,
having $WD_1(\tilde{p}_S^g, p_T^g)=0$ implies matching
of the class-conditionals, resulting in a
minimization of $\max_{k,z} w(z)S_k(z)$ (remind that $w(z)=1$ as mixture components $p_S^j$ and $p_T^j$ of $p_{\tilde S}^g$ and $p_{T}^g$ are both weighted by $p_T^{y=j}$ for all $j$, since we learn using $p_{\tilde S}^g$).

\section{Discussions}
\label{sec:discussions}

From a theoretical point of view, several works have pointed out the limitations of learning domain invariant representations. \cite{JohanssonSR19}, \cite{pmlr-v97-zhao19a} and \cite{pmlr-v97-wu19f} have
introduced some generalization bounds on the target error
that show the key role of label distribution and conditional
distribution shifts when learning invariant representations.
Importantly, \cite{pmlr-v97-zhao19a} and \cite{pmlr-v97-wu19f} have shown that in a label shift situation, minimizing source error while achieving invariant representation will tend to increase the target
error. In our work, we introduce an upper bound that clarifies the importance of learning invariant representations
that also align class-conditional representations in source and target domains.

Algorithmically, most related works are the one by \citet{pmlr-v97-wu19f} and \citet{combes2020domain}
that also address generalized target shift. The first approach does not seek at estimating label proportion but instead  allows flexibility in the alignment by using an assymetrically-relaxed distance.
 In the case of Wasserstein distance, the approach of \cite{pmlr-v97-wu19f} consists in reweighting the marginal
of the {source} distribution and in its dual form, their distance boils to
$$
WD_w(p_S,p_T) = \sup_{\|v\|_L  \leq 1}\mathbf{E}_{x \sim p_S} w(x)v(x) - 	\mathbf{E}_{x \sim p_T} v(x)  
$$
where $w(\cdot$) is actually a constant $\frac{1}{1 + \beta }$.
We can note that the adversarial loss we propose is a  general case
of this one. Indeed, in the above, the same amount of weighting applies
to all the samples of the source distribution. At the contrary, our reweighting scheme depends on the class-conditional probability and
their estimate target label proportion. Hence, we believe that our
 approach would adapt better to imbalance without the need to tune
$\beta$ (by validation for instance, which is hard in unsupervised domain
adaptation). The work of \citet{combes2020domain} and our differs only in
the way the weights $w(x)$ are estimated. In our case, we consider a theoretically supported and consistent estimation of the target label proportion, while they directly estimate $w(\cdot)$ by applying a technique tailored and grounded for problems without
class-conditional shifts. We will show in the experimental section that their estimator in some cases lead to poor generalization.

Still in  the context of reweighting, \cite{yan2017mind} proposed a weighted
Maximum Mean discrepancy distance for handling target shift in UDA. However, their
weights are estimated based on pseudo-labels obtained from the learned classifier
and thus, it is difficult to understand whether they provide accurate estimation
of label proportion even in simple setting. 
While their distance is MMD-transposed version of our weighted Wasserstein,
our approach applies to representation learning and is more theoretically grounded as the label proportion estimation
is based on sound algorithm with proven convergence guarantees (see below) and
our optimal assignment assumption provides guarantees on situations under which
class-conditional probability matching is correct.

The idea of matching moment of distributions have already been proven to be an effective for handling distribution mismatch. About ten years ago, \cite{huang2007correcting,gretton2009covariate,YuS12} already leveraged
such an idea for handling covariate shift by matching means of distributions in some reproducing
kernel Hilbert space. \cite{pmlr-v89-li19b} recycled the same idea for label proportion estimation and extended the idea to distribution matching. Interestingly, our approach
differs on its usage. While most above works employ mean matching for density ratio estimation
or for label proportion estimation, we use it as a mean for identifying displacement of 
class-conditional distributions through optimal assignment/transport. Hence, it allows
us to assign  estimated label proportion to the appropriate class.

For estimating the label proportion, we have proposed to learn a Gaussian mixture model of the target distribution. By doing so we are actually trying to solve a harder problem than necessary. However, once the target distribution estimation has
been evaluated and class-conditional probabilities being assigned from the source class, one
can use that Gaussian mixture model for labelling the target samples. 
Note however that Gaussian mixture learned by expectation-minimization can be hard to estimate
especially in high-dimension \cite{zhao2020statistical} and that the speed of convergence of the EM algorithm depends on smallest mixture weights \cite{NaimG12}. Hence, in high-dimension and/or highly imbalanced situations, one may get a poor estimate of the target distribution.
Nonetheless, one can consider other non-EM approach \cite{kannan2005spectral,arora2005learning}.
Hence, in practice, we can expect the approach GMM estimation and OT-based matching to be a
strong baseline  in low-dimension and well-clustered mixtures setting but to break in high-dimension one.

{
}

\section{Numerical Experiments}
\label{sec:numerical_experiments}
We present in this section some experimental analyses of the proposed algorithm on
a toy dataset as well as on real-world visual domain adaptation problems. The code
for reproducing part of the experiments is available at \url{https://github.com/arakotom/mars_domain_adaptation}.

\subsection{Experimental Setup}
Our goal is to show that among algorithms tailored for handling generalized target shift,
our method is the best performing one (on average). Hence, we  compare with two very recent methods designed for generalized target shift and with two  domain adaptation algorithms tailored for covariate shift for sanity check.

 As
a baseline, we consider a model, denoted as {Source}
 trained for $f$ and $g$ on the source examples and
tested without adaptation on the target examples. Two other competitors use respectively an adversarial domain learning \cite{ganin2016domain}
and the Wasserstein distance \cite{shen2018wasserstein} computed  in the dual as distances for measuring
discrepancy between $p_S$ and $p_T$, denoted as DANN and $\text{WD}_{\beta=0}$. We consider the model proposed by \citet{pmlr-v97-wu19f} and \citet{combes2020domain} as competing algorithms able to cope with generalized target shift.
For this former approach, we use the asymmetrically-relaxed Wasserstein distance so as to
make it similar to our approach and also report results for different values of the relaxation $\beta$. 
This model is named $\text{WD}_{\beta}$ with $\beta \geq 1$. 
The \citet{combes2020domain}'s method, named IW-WD (for importance weighted Wasserstein distance) solves the same learning problem as ours and differs only on the way the ratio $w(x)$ is estimated.
Our approaches are denoted as MARSc or MARSg respectively when
estimating proportion by hierarchical clustering or by Gaussian mixtures. 
All methods differ only in the metric used for computing the distance
between marginal distributions and most of them except DANN use a Wassertein distance. The difference essentially relies on the reweighting strategy of
the source samples.  For all models, learning rate and  the hyperparameter $\lambda$ in Equation \ref{eq:optimisation2} have been chosen based on a reverse cross-validation strategy. The metric that we have used for comparison is the balanced accuracy (the average recall obtained on each class) which is better suited for imbalanced problems \citep{brodersen2010balanced}.
All presented results have been obtained as averages over $20$ runs.

\begin{figure*}[t]
	\begin{center}
		~\hfill
		\includegraphics[width=0.32\linewidth]{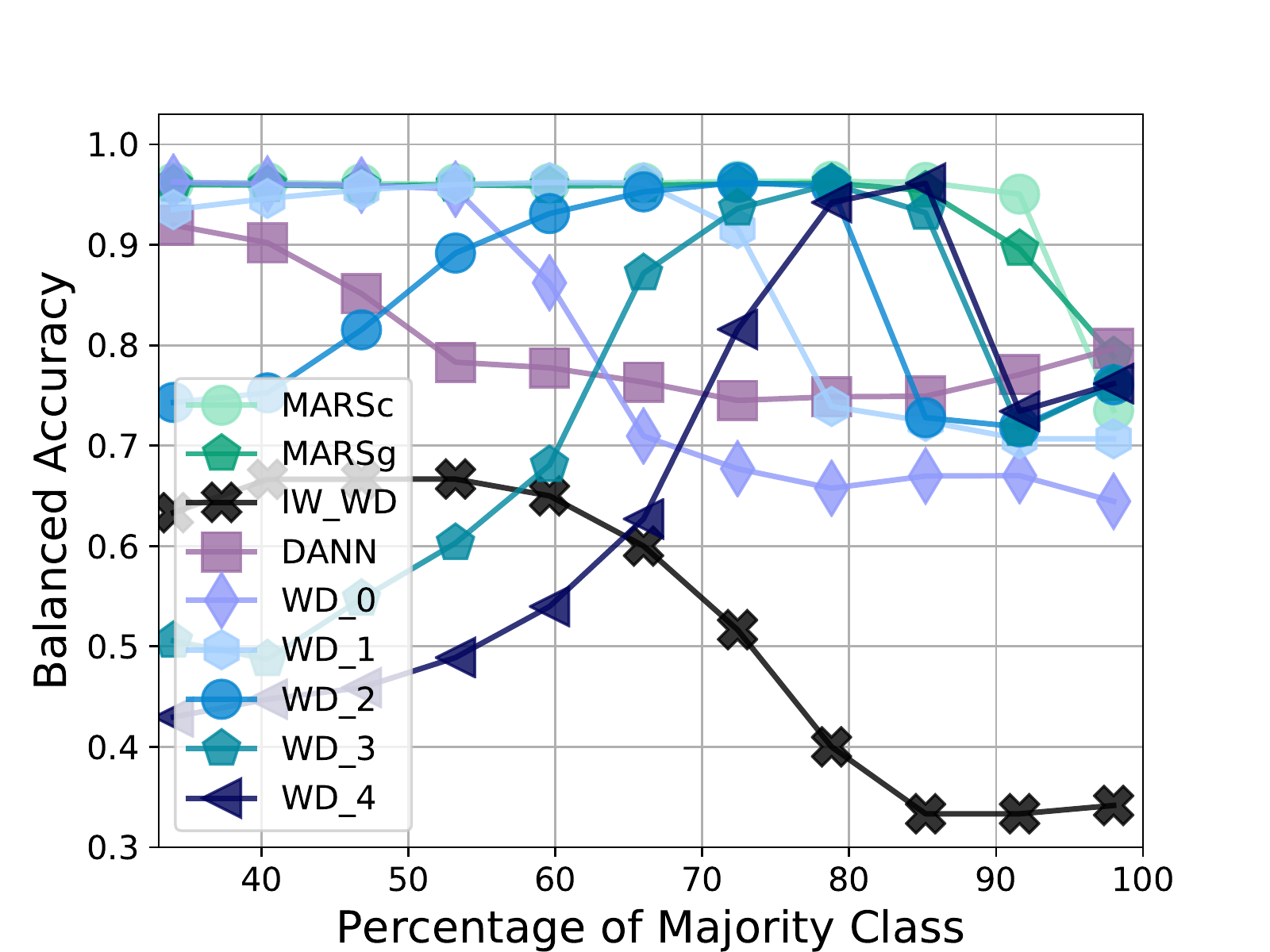}
		\includegraphics[width=0.32\linewidth]{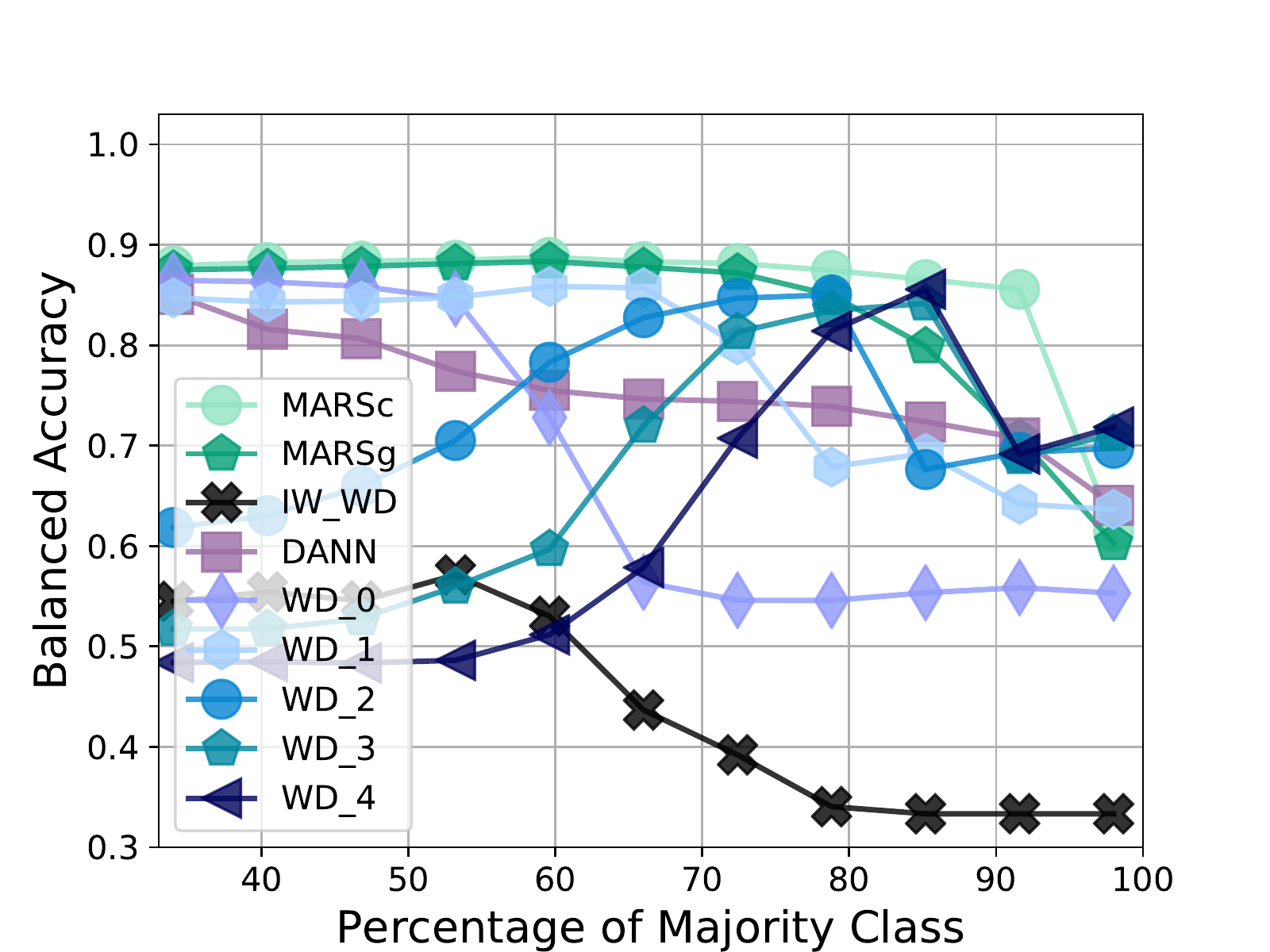}
		\includegraphics[width=0.32\linewidth]{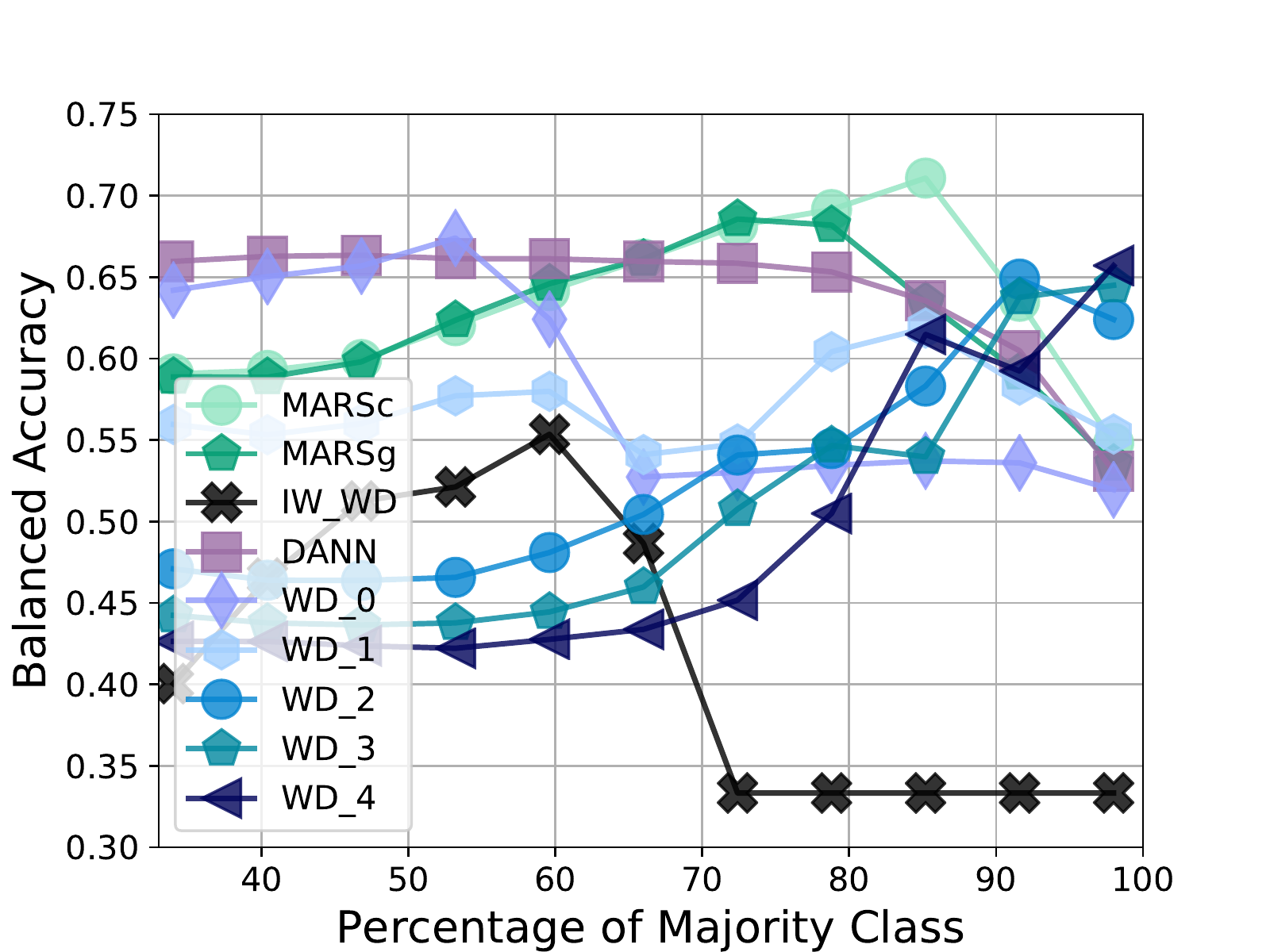}
		\hfill~
		\caption{Performance of the compared algorithms for three different
			covariance matrices of the Gaussians composing the toy dataset with respect to the 		imbalance. The x-axis is given with respect to the percentage of majority class which is
			the class $1$.	(left) Low-error setting. (middle) mid-error setting.
			(right) high-error setting. Example of the source and target samples for the different cases are provided in the supplementary material.\label{fig:toyimbalance}}
	\end{center}
\end{figure*}

\begin{figure*}[t]
	\begin{center}
		~\hfill
		\includegraphics[width=0.3\linewidth]{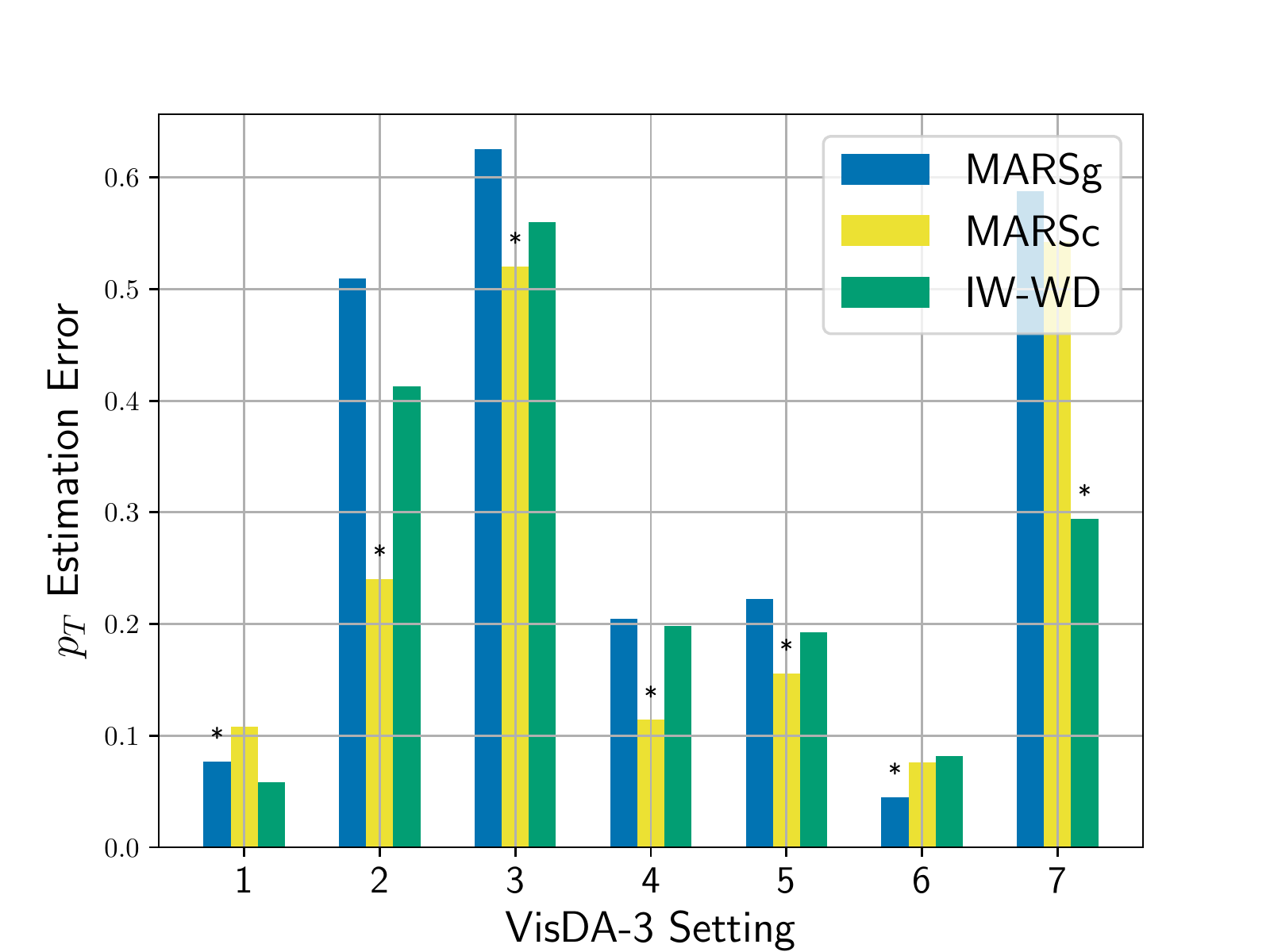}~\hfill~
		\includegraphics[width=0.3\linewidth]{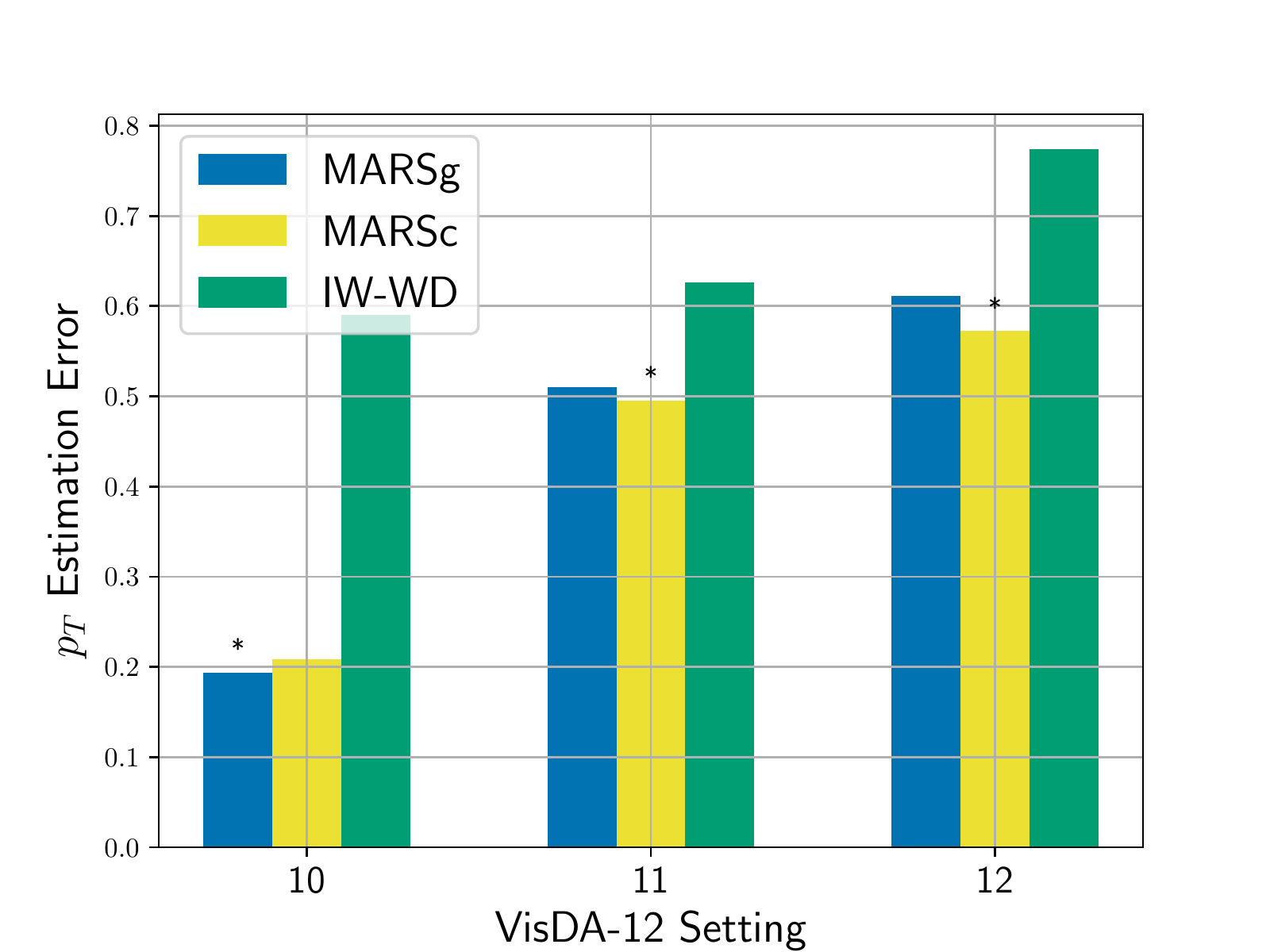}~\hfill~ 
		~\hfill~\includegraphics[width=0.3\linewidth]{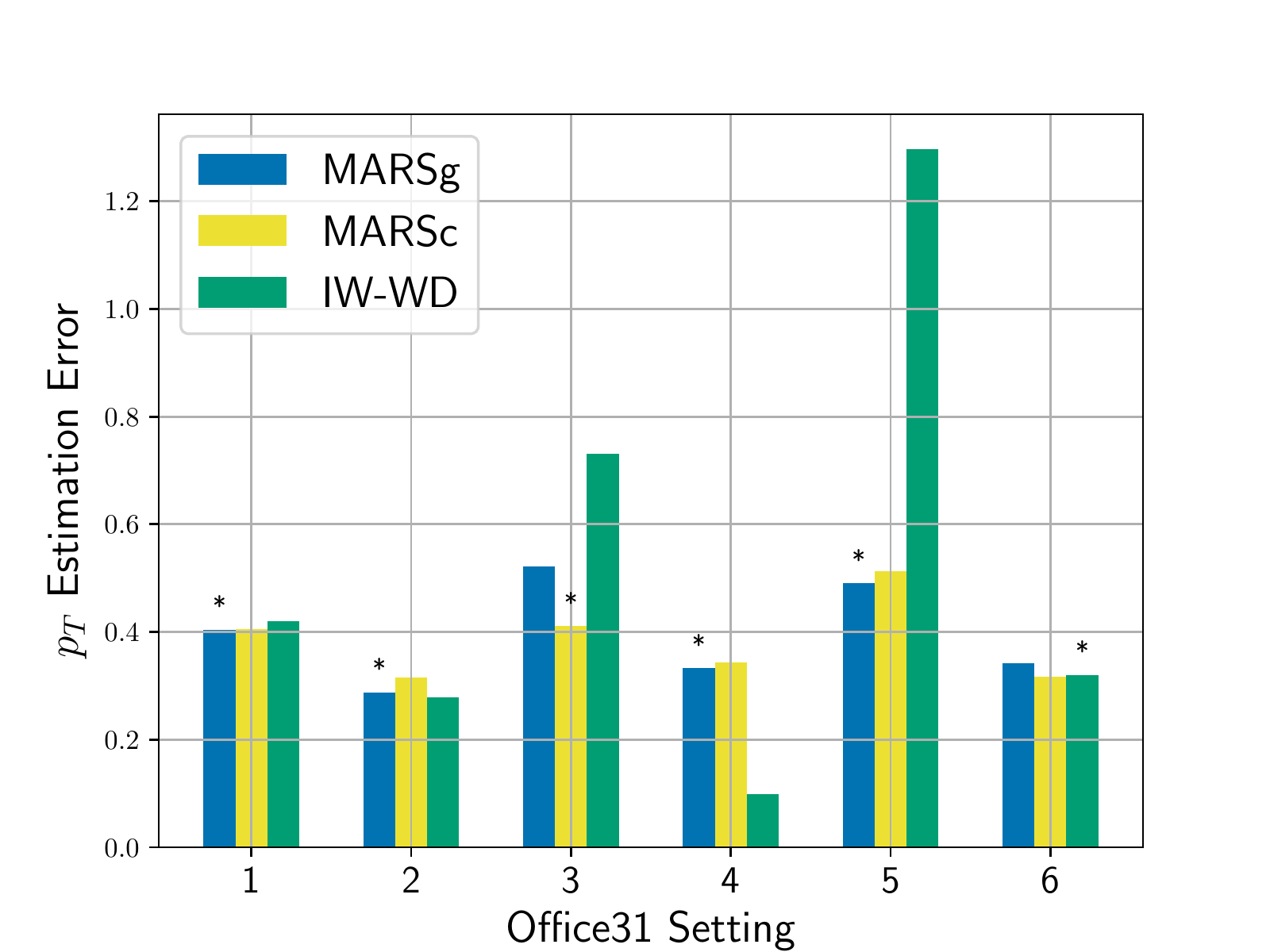}
		\hfill~
		\caption{Examples of $\ell_1$ norm error of  estimated label proportion. We have reported the performance of our two methods (MARSg and MARSc) as well as the performance of IW-WD. The three panels are related to the (left) VisDA-3, (middle) VisDA-12, (right) Office 31 and the different experimental imbalance settings (see Table \ref{tab:dataconfig}). We have also reported, with a `*' on top, among the three approaches, the best performing one in term of balanced accuracy.
			We note that MARSc provides better estimation than IW-WD on $12$ out of $16$ experiments. Note also the  correlation between better $\p_T$ estimation and  accuracy.
			\label{fig:proportion}}
	\end{center}
\end{figure*}

\begin{figure}
		\includegraphics[width=5.cm]{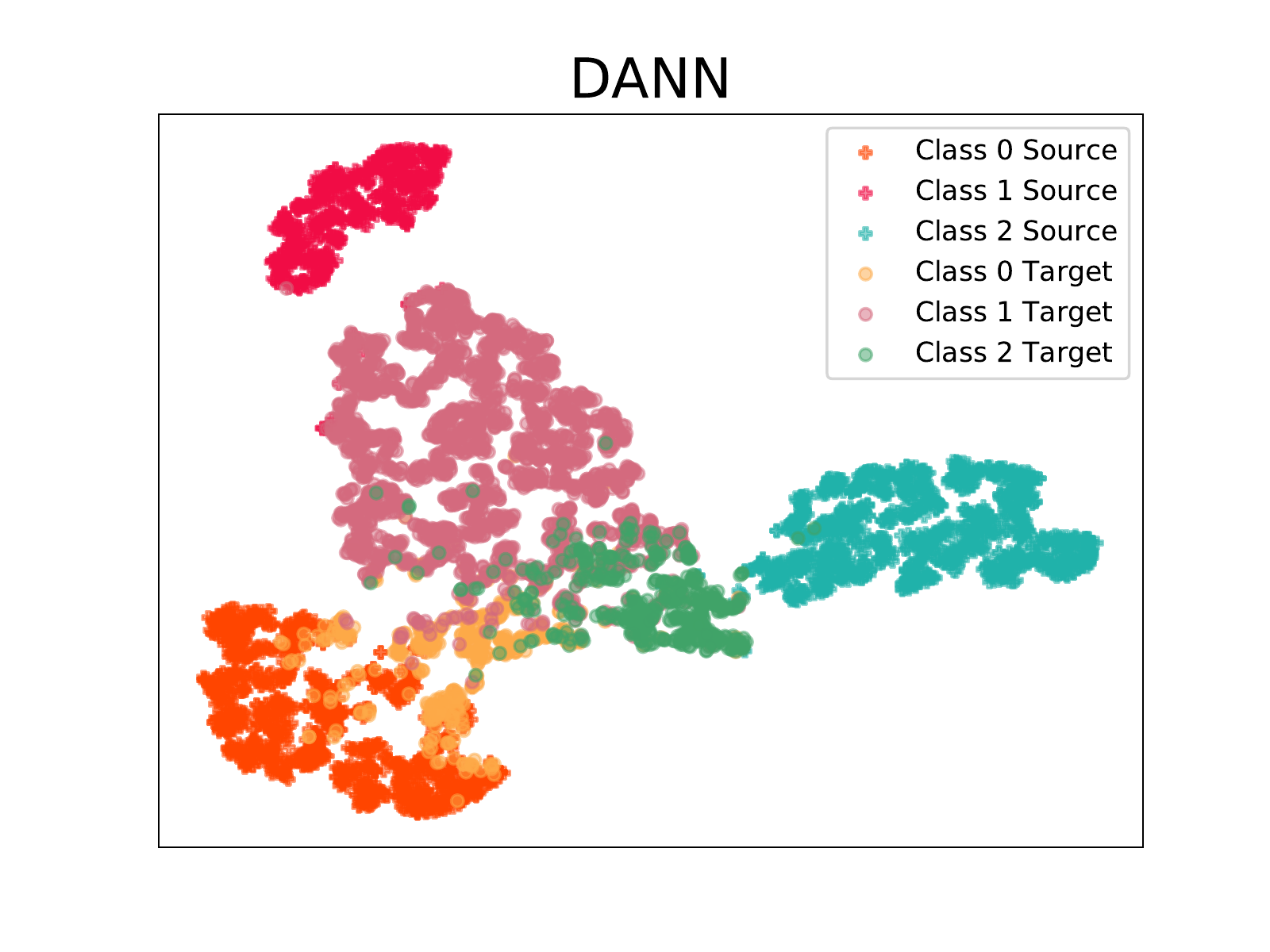}
	~\hfill~\includegraphics[width=5.cm]{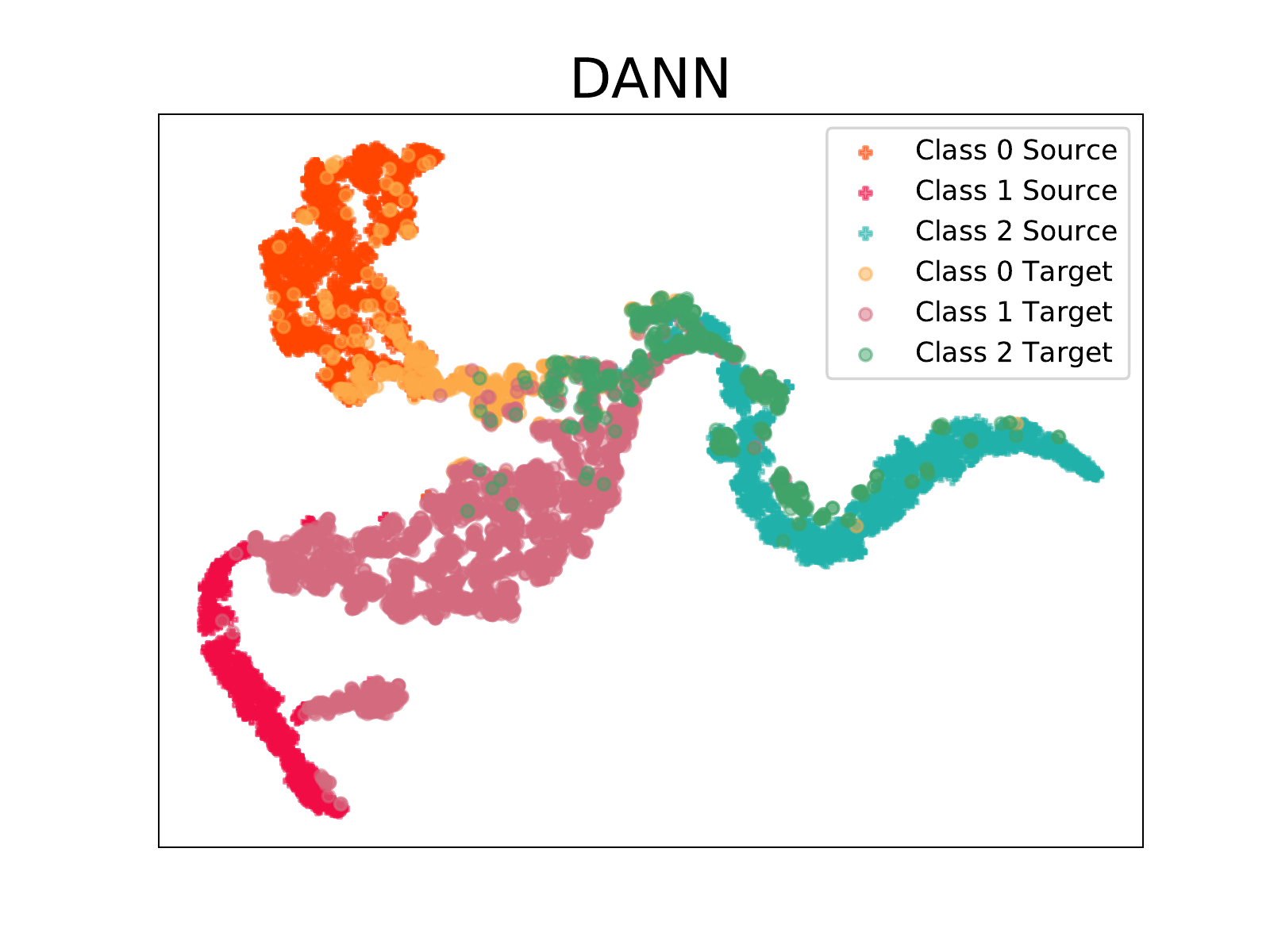}	\hfill~\\
		~\hfill
	\includegraphics[width=5.cm]{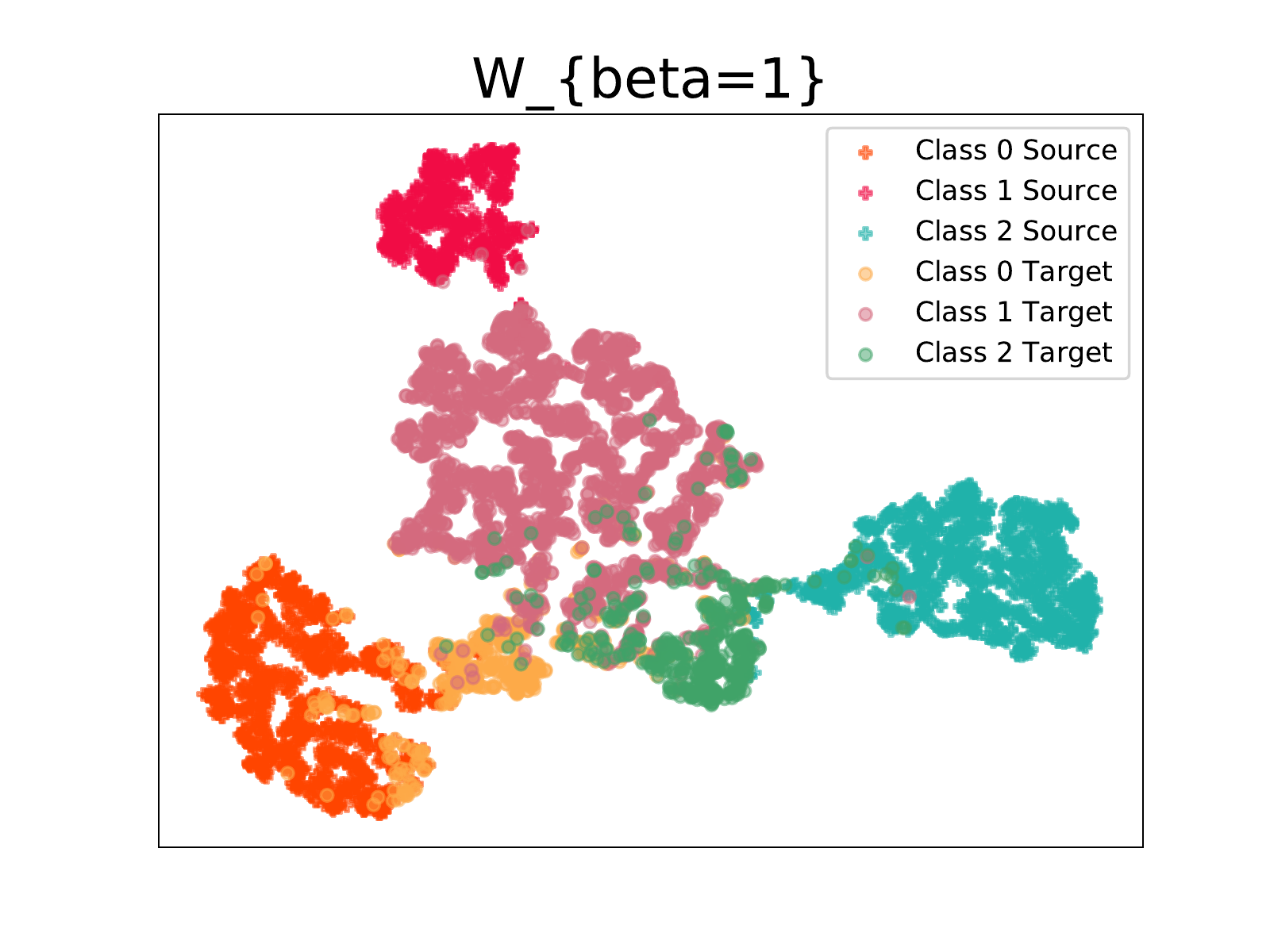}
	~\hfill~\includegraphics[width=5.cm]{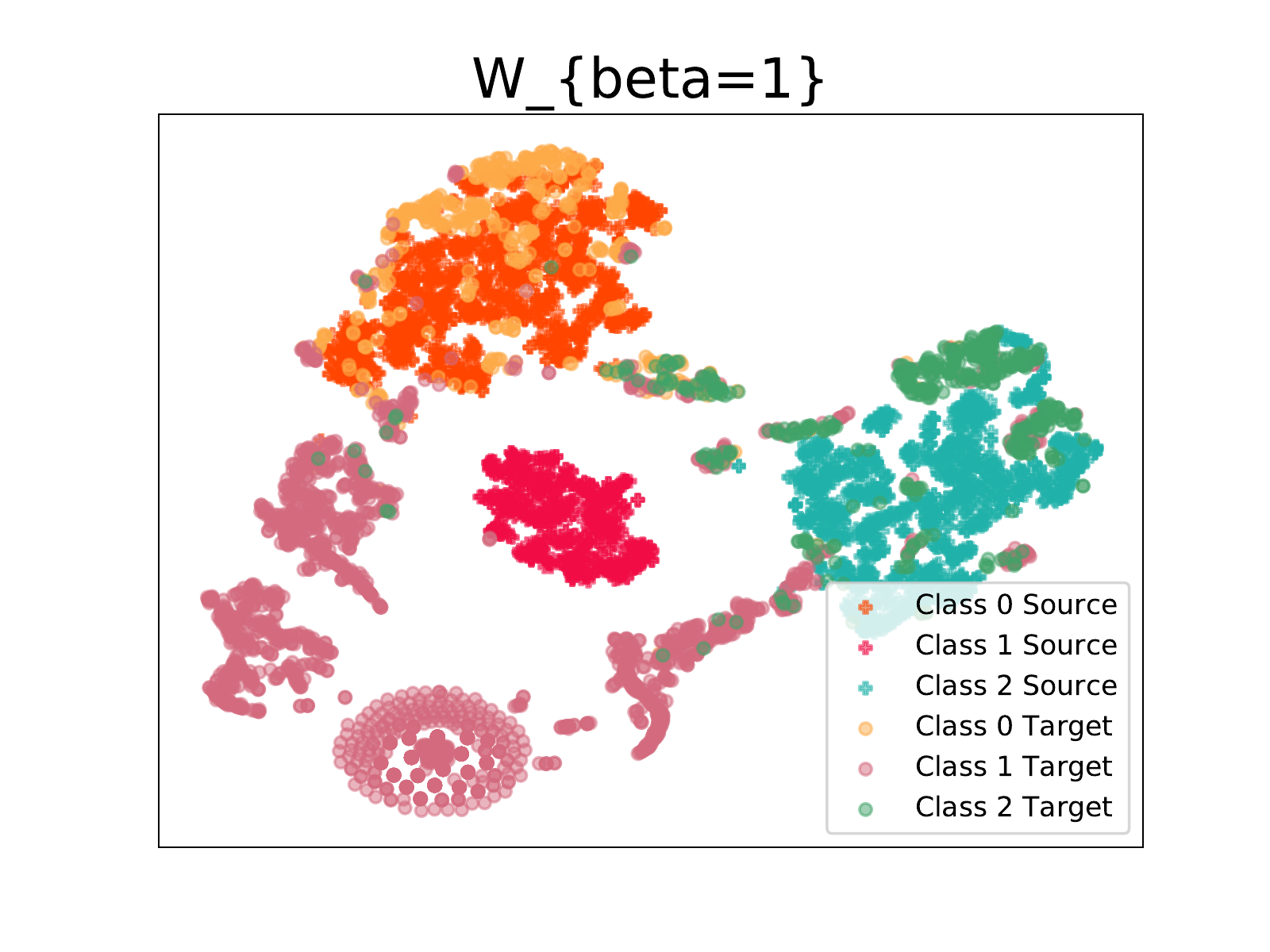}	\hfill~\\
		~\hfill
	\includegraphics[width=5.cm]{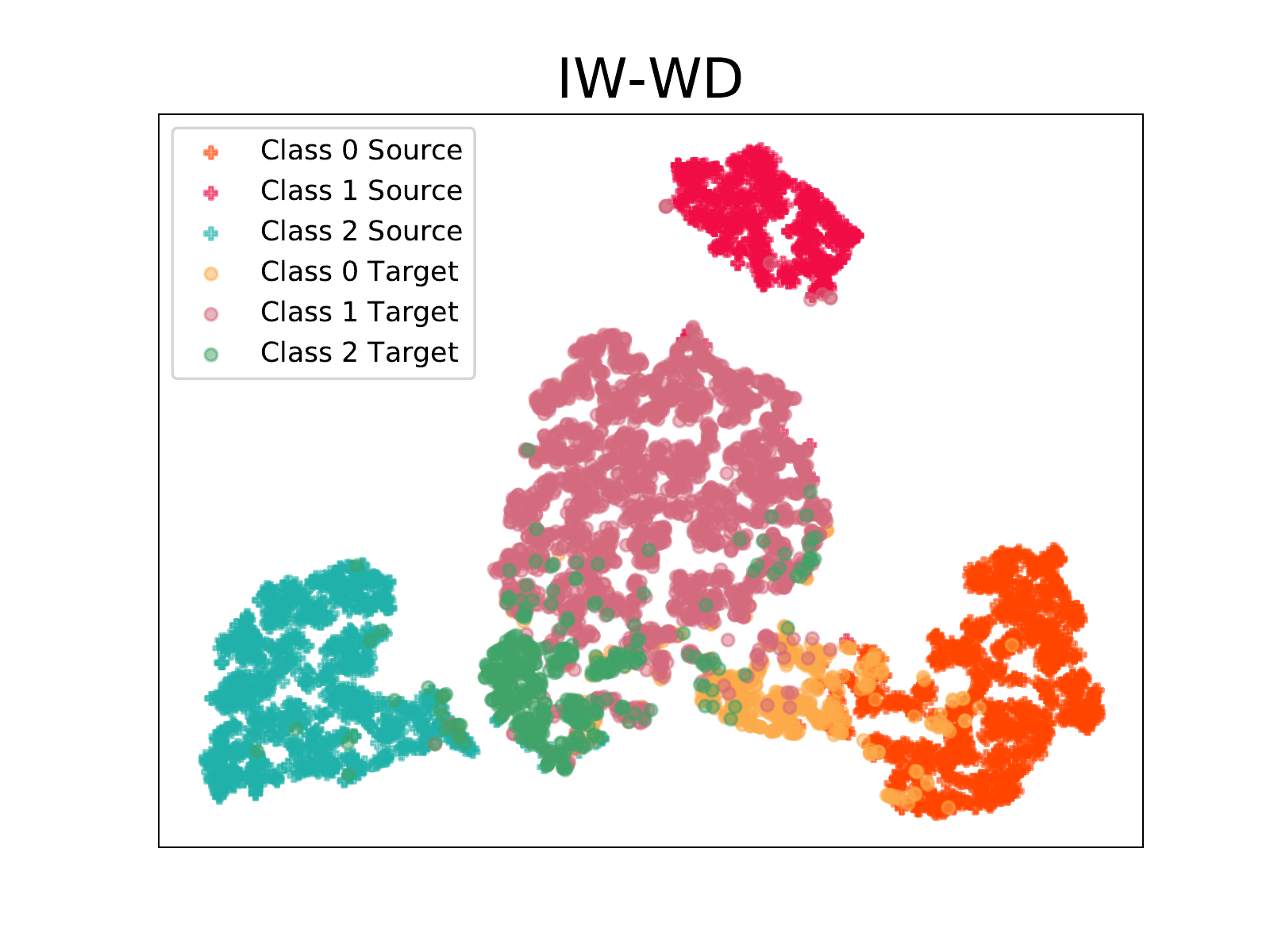}
	~\hfill~\includegraphics[width=5.cm]{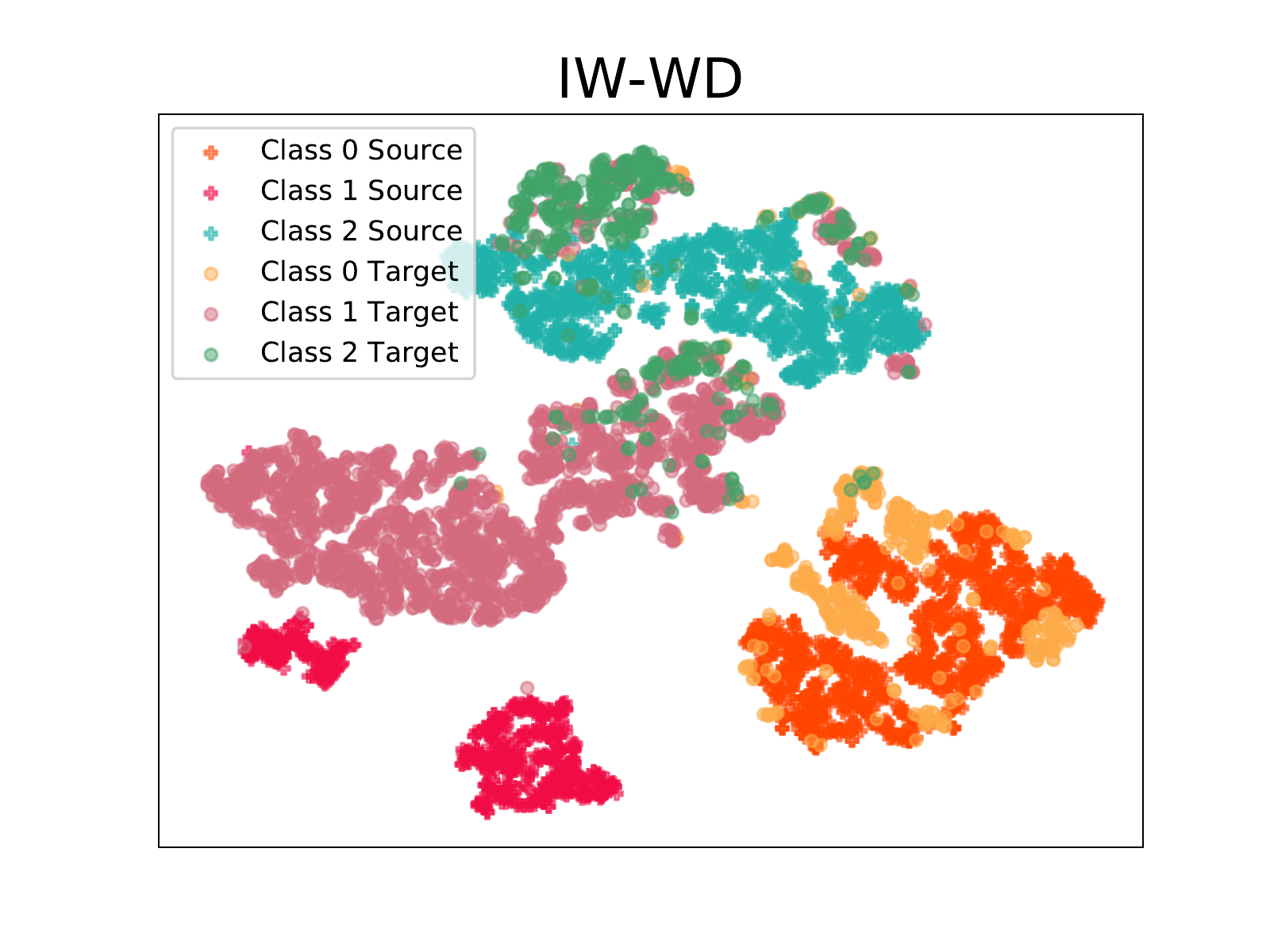}	\hfill~\\
		~\hfill
	\includegraphics[width=5.cm]{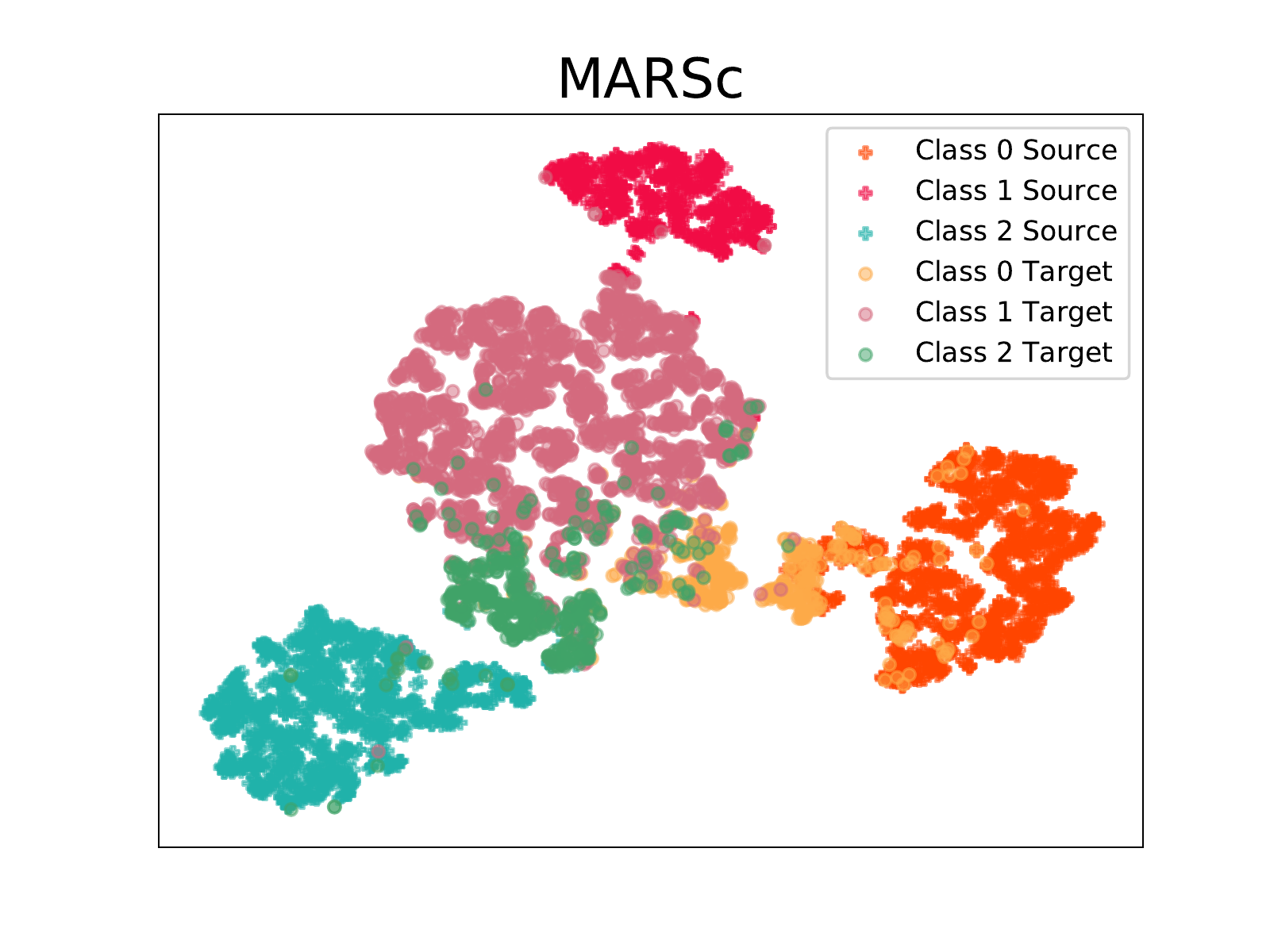}
	~\hfill~\includegraphics[width=5.cm]{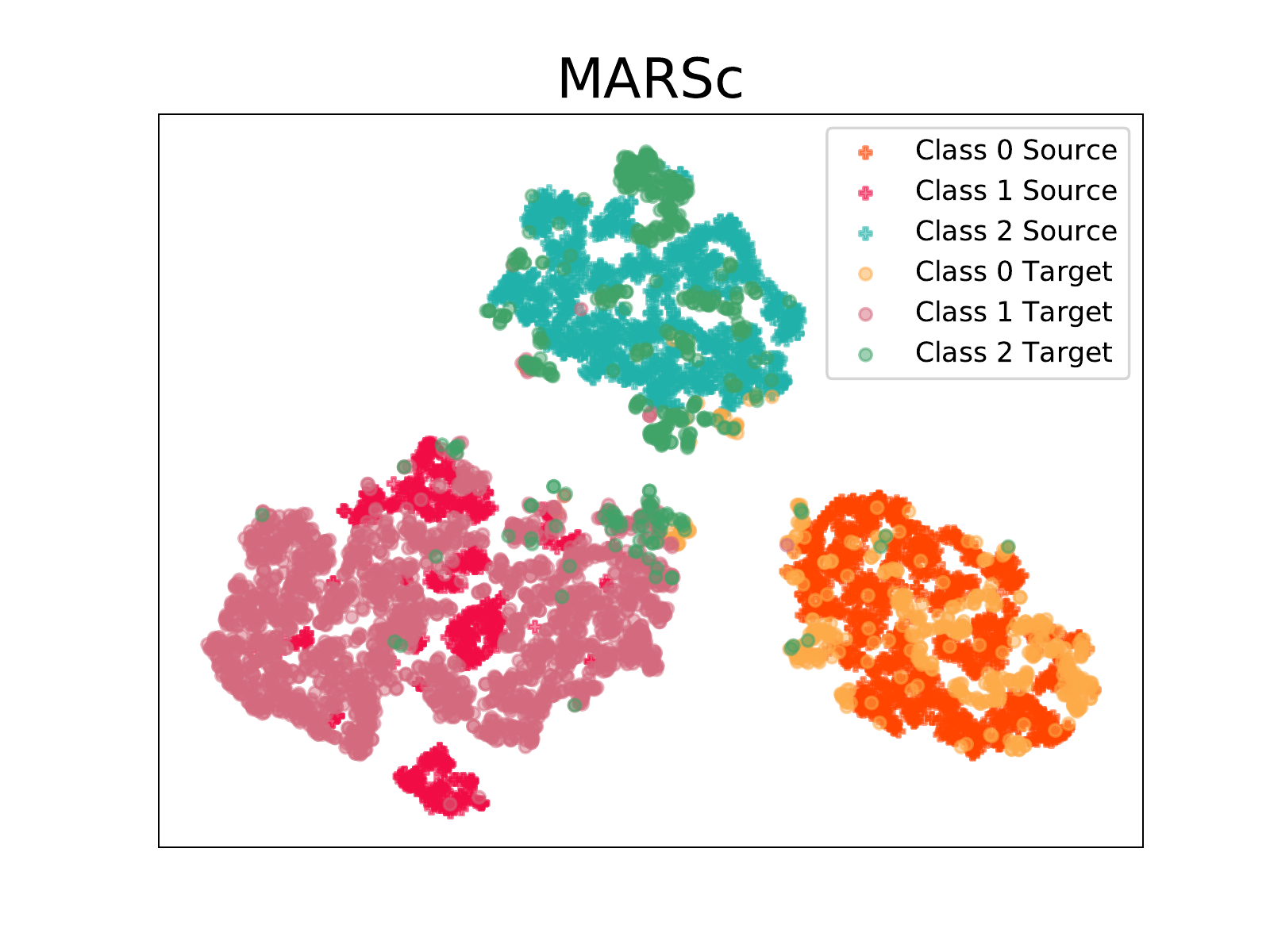}	\hfill~\\
	\caption{\emph{t-sne} embeddings of the target sample for the VisDA-3 problem and imbalance setting $2$ ($\p_S=[0.4,0.2,0.4]$ and $\p_T=[0.2,0.6,0.2]$). The columns depict the 
		embeddings obtained (left) after training on the source data without adaptation for about $10$ iterations, which is sufficient for $0$ training error. (right) after adaptation by minimizing the appropriate
		discrepancy loss between marginal distributions.
		From top to bottom, we have : (first-row) DANN, (second-row) WD$_{\beta=1}$, (third-row), IW-WD (last row) MARSc. From the
		right column, we note how DANN and WD$_{\beta=1}$ struggle in aligning the class conditionals, especially those of Class 1, which is the class that varies the most in term of label proportion. IW-WD manages to aligns the classes ``$0$" and ``$2$" but is not able
		to correctly match the class ``$1$". Instead, our MARSc approach achieves high performance and correctly aligns the class conditionals, although some few examples seem to be mis-classified. Importantly, we can remark from the left column that for this example, before alignment, the embeddings seem to satisfy our Proposition 1 hypothesis. At the contrary, the assumption needed for correctly estimating $\p_T$ for IW-WD is not satisfied, justifying
		thus the good and poor performance of those models.
		\label{fig:embeddings}}
	
\end{figure}

\subsection{Toy Dataset} The toy dataset is a 3-class problem in which
class-conditional probabilities are Gaussian distributions.
For the source distribution, we fix the mean and the covariance matrix of each
of the three Gaussians and for the target, we simply shift the means (by
a fixed translation). We have carried out two sets of experiments where we
have  fixed the shift and modified the label proportion imbalance and
another one with fixed imbalance and increasing shift. 
For space reasons, we have deported to the supplementary the results of the latter.
 Figure \ref{fig:toyimbalance} show how models perform for
varying imbalance and fixed shift. The plots nicely show what we expect. 
DANN performs worse as the imbalance increases. 
{$\text{WD}_{\beta}$ works well for all balancing but its parameter $\beta$ needs to increase with the imbalance level.}
Because of the shift in class-conditional probabilities, IW-WD is not able to properly estimate the importance weights and fails.
Our approaches are adaptive to the
imbalance and perform very well over a large range for both a low-noise and
mid-noise setting (examples of how the Gaussians are mixed are provided in the supplementary material). For the hardest problem (most-right panel), all models
have difficulties and achieve only a balanced accuracy of $0.67$ over some range of imbalance. 
Note that for this low-dimension toy problem, as expected, the approach GMM and OT-based matching achieves the best performance as reported in the supplementary material.

\subsection{Digits, VisDA  and Office}

We present some UDA experiments on computer vision datasets \citep{peng2017visda,venkateswara2017Deep}, with different imbalanced settings. Details of problem configurations as well as model architecture and training procedure can be found in the appendix.

Our first result provides an illustration in Figure 
\ref{fig:embeddings} of the latent representation we obtain for the VisDA problem after  training
on the source domain only and after convergence of the different DA algorithms. We first note that for this problem, the assumptions for correct matching seem to hold and this leads to very good visual matching of class-conditionals for MARS. 

Table \ref{tab:results} reports the averaged  balanced accuracy achieved by the different models
for only a fairly chosen subset of problems. The full table is in the supplementary. Results presented here are not comparable to results available in the literature as they mostly consider covariate shift DA (hence with balanced proportions). 
For these subsets of problems, our approaches yield the best average ranking. They perform better than competitors except on the 
MNIST-MNISTM problems where the change in distribution might violate our assumptions. Figure \ref{fig:proportion} presents some quantitative results label proportion estimation in the target domain between our method and 
IW-WD. We show that MARSc provides better estimation than this competitor 12 out of 16 experiments. As the key issue in  generalized target shift problem is the ability to estimate accurately
the importance weight or the target label proportion, we believe that the
learnt latent representation fairly satisfies our OT hypothesis leading
to good performance.

\begin{table*}
	\small
			\caption{Table of averaged \textbf{balanced accuracy} for the compared models and different domain adaptation problems and  label proportion imbalance settings.  Reported in bold are the best performances as well as other methods which achieve 
		performance that are statistically similar according to a Wilcoxon signrank test with $p=0.01$. Last lines present the summary of $34$ experiments. \#Win includes the statistical ties.\label{tab:results}}
	\resizebox{\linewidth}{!}{
		\begin{tabular}{lcccccccccc} 
			\\\hline
			Setting & Source    & DANN    & $\text{WD}_{\beta=0}$    &$\text{WD}_{\beta=1}$    & $\text{WD}_{\beta=2}$    &$\text{WD}_{\beta=3}$    & $\text{WD}_{\beta=4}$   & IW-WD & MARSg    & MARSc  \\\hline 
			\multicolumn{10}{c}{MNIST-USPS 10 modes} \\\hline
						Balanced & 76.9$\pm$3.7  & 79.7$\pm$3.5  & 93.7$\pm$0.7  & 74.3$\pm$4.3  & 51.3$\pm$4.0  & 76.6$\pm$3.3  & 71.9$\pm$5.7  & 95.3$\pm$0.4  & \textbf{95.6$\pm$0.7}  & \textbf{95.6$\pm$1.0}  \\
						Mid   & 80.4$\pm$3.1  & 78.7$\pm$3.0  & 94.3$\pm$0.7  & 75.4$\pm$3.4  & 55.6$\pm$4.3  & 79.0$\pm$3.1  & 72.3$\pm$4.2  & \textbf{95.6$\pm$0.5}  & 89.7$\pm$2.3  & 90.4$\pm$2.6 \\
						High & 78.1$\pm$4.9  & 81.8$\pm$4.0  & \textbf{93.9$\pm$1.1}  & 87.4$\pm$1.7  & 83.8$\pm$5.2  & 85.7$\pm$2.5  & 83.6$\pm$3.0  & \textbf{94.1$\pm$1.0}  & 88.3$\pm$1.5  & 89.7$\pm$2.3  \\
			
			\hline
			\multicolumn{10}{c}{USPS-MNIST 10 modes} \\\hline
						Balanced 	 & 77.0$\pm$2.6  & 80.5$\pm$2.2  & 73.4$\pm$2.8  & 66.7$\pm$2.9  & 49.9$\pm$2.8  & 55.8$\pm$2.9  & 52.1$\pm$3.5  & 80.5$\pm$2.2  & \textbf{84.6$\pm$1.7}  & \textbf{85.5$\pm$2.1} \\

						Mid  	 & \textbf{79.5$\pm$2.8}  & \textbf{78.9$\pm$1.8}  & 75.8$\pm$1.6  & 63.3$\pm$2.3  & 53.2$\pm$2.8  & 47.2$\pm$2.4  & 48.3$\pm$2.9  & \textbf{78.4$\pm$3.5}  &\textbf{ 79.7$\pm$3.6}  & \textbf{78.5$\pm$2.5} \\ 
			
						High  	 & \textbf{78.5$\pm$2.4}  & \textbf{77.8$\pm$2.0}  & \textbf{76.1$\pm$2.7}  & 63.0$\pm$3.3  & 57.6$\pm$4.8  & 51.2$\pm$4.4  & 49.3$\pm$3.3  & 71.5$\pm$4.7  & 75.6$\pm$1.8  & \textbf{77.1$\pm$2.4} \\\hline
			\multicolumn{10}{c}{MNIST-MNISTM 10 modes} \\\hline
			Setting 1      & 58.3$\pm$1.3  & \textbf{61.2$\pm$1.1}  & 57.4$\pm$1.7  & 50.2$\pm$4.4  & 47.0$\pm$2.0  & 57.9$\pm$1.1  & 60.0$\pm$1.3  & \textbf{63.1$\pm$3.1}  & 58.1$\pm$2.3  & 56.6$\pm$4.6 \\
			Setting 2   & 60.0$\pm$1.1  & 61.1$\pm$1.0  & 58.1$\pm$1.4  & 53.4$\pm$3.5  & 48.6$\pm$2.4  & 59.7$\pm$0.7  & 58.1$\pm$0.8  & 65.0$\pm$3.5  & 57.7$\pm$2.3  & 55.7$\pm$2.1 \\
			Setting 3       & 58.1$\pm$1.2  & \textbf{60.4$\pm$1.4}  & 57.7$\pm$1.2  & 47.7$\pm$4.9  & 42.2$\pm$7.3  & 57.1$\pm$1.0  & 53.5$\pm$1.1  & \rev{52.5$\pm$14.8}  & 53.7$\pm$7.2  & 53.7$\pm$3.3 \\\hline
												\multicolumn{10}{c}{VisdDA 3 modes} \\\hline
			
			setting 1        & 79.3$\pm$4.3  & 78.9$\pm$9.1  & {91.8$\pm$0.7}  & 73.8$\pm$2.0  & 61.7$\pm$2.2  & 65.6$\pm$2.7  & 58.6$\pm$2.6  & \textbf{94.1$\pm$0.6}  & {92.5$\pm$1.2}  & {92.1$\pm$1.8}  \\
			setting 4        & 80.2$\pm$5.3  & 75.5$\pm$9.3  & 72.8$\pm$1.2  & 86.9$\pm$7.5  & 86.8$\pm$1.2  & 80.2$\pm$6.9  & 75.7$\pm$2.0  & 85.9$\pm$5.7  & 87.7$\pm$3.0  & \textbf{91.3$\pm$4.8} \\
			setting 2        & 81.5$\pm$3.5  & \rev{68.5$\pm$14.7}  & 68.8$\pm$1.3  & 84.5$\pm$1.2  & \textbf{93.2$\pm$0.4}  & \rev{73.7$\pm$14.2}  & 60.7$\pm$0.9  & \rev{78.7$\pm$10.8}  & 84.0$\pm$4.3  & \textbf{91.8$\pm$3.4}  \\
			setting 3        & 78.4$\pm$3.2  & \rev{59.0$\pm$15.9}  & 64.1$\pm$1.9  & 79.2$\pm$0.8  & 77.1$\pm$10.3  & 90.0$\pm$0.5  & \textbf{94.4$\pm$0.3}  & 78.0$\pm$9.3  & 75.7$\pm$4.1  & \rev{73.9$\pm$13.2} \\
			setting 5        & 83.5$\pm$3.5  & \rev{80.9$\pm$14.5}  & 63.9$\pm$0.6  & 73.7$\pm$7.3  & 50.9$\pm$1.1  & 76.5$\pm$6.7  & 59.3$\pm$1.0  & \textbf{90.4$\pm$3.6}  & \textbf{89.0$\pm$0.9}  & \textbf{89.0$\pm$3.5} \\
			setting 6        & 80.9$\pm$4.2  & \rev{54.8$\pm$19.8}  & 45.3$\pm$2.4  & 63.7$\pm$5.1  & 67.1$\pm$6.1  & 42.9$\pm$11  & 62.2$\pm$1.4  & \textbf{94.4$\pm$1.0}  & \textbf{93.7$\pm$0.4}  & \textbf{93.9$\pm$1.0} \\
			setting 7        & 79.2$\pm$3.7  & 42.9$\pm$2.5  & 57.5$\pm$1.5  & 55.4$\pm$2.0  & 50.2$\pm$4.3  & 43.7$\pm$8.3  & 62.5$\pm$0.8  & \textbf{88.5$\pm$4.9}  & 78.6$\pm$3.2  & \textbf{82.3$\pm$7.5} \\\hline
			
			\multicolumn{10}{c}{VisdDA 12 modes} \\\hline
			setting 1       & 41.9$\pm$1.5  & 52.8$\pm$2.1  & 45.8$\pm$4.3  & 44.2$\pm$3.0  & 35.5$\pm$4.6  & 41.0$\pm$3.0  & 37.6$\pm$3.4  & 50.4$\pm$2.3  & 53.3$\pm$0.9  & \textbf{55.1$\pm$1.6}\\
			setting 2      & 41.8$\pm$1.5  & 50.8$\pm$1.6  & 45.7$\pm$8.9  & 40.5$\pm$4.8  & 36.2$\pm$5.0  & 36.1$\pm$4.6  & 31.9$\pm$5.7  & 48.6$\pm$1.8  & 53.1$\pm$1.6  & \textbf{55.3$\pm$1.6} \\
			setting 3     & 40.6$\pm$4.3  & 49.2$\pm$1.3  & 47.1$\pm$1.6  & 42.1$\pm$3.0  & 36.3$\pm$4.4  & 37.3$\pm$3.5  & 35.0$\pm$5.4  & 46.6$\pm$1.3  & 50.8$\pm$1.6  & \textbf{52.1$\pm$1.2}  \\\hline
						\multicolumn{10}{c}{Office 31} \\\hline
			A - D       & 73.7$\pm$1.4  & 74.3$\pm$1.8  & \textbf{77.2$\pm$0.7}  & 65.1$\pm$2.0  & 62.7$\pm$2.6  & 71.5$\pm$1.2  & 63.9$\pm$1.1  & 75.7$\pm$1.6  & 76.1$\pm$0.9  & \textbf{78.2$\pm$1.3}  \\
			D - W        & 83.7$\pm$1.1  & 81.9$\pm$1.5  & 82.6$\pm$0.6  & 83.5$\pm$0.8  & 82.8$\pm$0.7  & 80.1$\pm$0.5  & \textbf{87.1$\pm$0.9}  & 78.9$\pm$1.5  & \textbf{86.3$\pm$0.6}  & 86.2$\pm$0.8 \\
			W - A        & 54.1$\pm$0.9  & 52.2$\pm$1.0  & 48.9$\pm$0.4  & 56.8$\pm$0.4  & 53.0$\pm$0.5  & 58.8$\pm$0.4  & 54.9$\pm$0.5  & 52.2$\pm$0.7  & \textbf{60.7$\pm$0.8}  & 55.2$\pm$0.8 \\
			W - D       & 92.8$\pm$0.9  & 87.8$\pm$1.4  & 95.1$\pm$0.3  & 93.1$\pm$0.5  & 87.6$\pm$0.9  & 94.7$\pm$0.6  & 91.2$\pm$0.6  & \textbf{97.0$\pm$0.9}  & 95.1$\pm$0.8  & 93.8$\pm$0.6 \\
			D - A       & 52.5$\pm$0.9  & 48.1$\pm$1.2  & 49.8$\pm$0.4  & 48.8$\pm$0.5  & 50.1$\pm$0.4  & 50.3$\pm$0.7  & 50.8$\pm$0.5  & 41.4$\pm$1.8  & \textbf{54.7$\pm$0.9}  &\textbf{55.0$\pm$0.9} \\
			A - W      & 67.5$\pm$1.5  & 70.2$\pm$1.0  & 67.1$\pm$0.6  & 60.6$\pm$2.1  & 52.9$\pm$1.4  & 64.0$\pm$1.3  & 59.7$\pm$0.8  & 68.8$\pm$1.6  & \textbf{73.1$\pm$1.5}  & \textbf{71.9$\pm$1.2} \\\hline
			\multicolumn{10}{c}{Office Home} \\\hline
			Art - Clip      & 37.7$\pm$0.7  & 36.8$\pm$0.6  & 33.4$\pm$1.2  & 31.4$\pm$1.6  & 27.1$\pm$1.6  & 31.6$\pm$5.2  & 29.3$\pm$6.6  & 37.7$\pm$0.6  & 37.6$\pm$0.5  & \textbf{38.65$\pm$0.5} \\
			Art - Product       & 49.7$\pm$0.9  & 50.0$\pm$0.9  & 39.4$\pm$3.6  & 38.8$\pm$2.3  & 35.1$\pm$2.3  & 35.1$\pm$3.4  & 32.9$\pm$3.6  & 49.0$\pm$0.3  & \textbf{55.3$\pm$0.7}  & 52.2$\pm$0.4 \\
			Art - Real       & 58.2$\pm$1.0  & 53.7$\pm$0.5  & 51.1$\pm$2.3  & 50.4$\pm$1.8  & 46.4$\pm$2.4  & 51.5$\pm$4.5  & 45.3$\pm$11.0  & 57.7$\pm$0.7  &\textbf{ 63.88$\pm$0.5}  & 58.8$\pm$0.7 \\
			Clip - Art      & 35.3$\pm$1.4  & 35.7$\pm$1.5  & 28.9$\pm$2.9  & 23.1$\pm$2.0  & 18.4$\pm$1.5  & 22.0$\pm$3.1  & 20.4$\pm$2.3  & 28.7$\pm$1.2  & \textbf{41.2$\pm$0.6}  & \textbf{40.7$\pm$0.8} \\
			Clip - Product       &\textbf{51.9$\pm$1.3}  & \textbf{52.1$\pm$0.8}  & 39.2$\pm$7.9  & 39.3$\pm$2.6  & 34.7$\pm$1.9  & 39.6$\pm$2.8  & 39.5$\pm$2.9  & 34.5$\pm$2.1  & 51.7$\pm$0.5  & \textbf{52.1$\pm$0.5 } \\
			Clip - Real       & 50.7$\pm$1.2  & 51.4$\pm$1.0  & 43.2$\pm$2.2  & 40.1$\pm$2.1  & 32.7$\pm$1.4  & 39.2$\pm$2.4  & 35.8$\pm$2.8  & 35.7$\pm$1.1  & 54.0$\pm$0.3  & \textbf{56.6$\pm$0.5}  \\
			Product - Art       & \textbf{39.6$\pm$1.6}  & 39.5$\pm$1.5  & 39.2$\pm$1.0  & 36.1$\pm$1.0  & 38.8$\pm$1.1  & 39.5$\pm$0.6  & 38.2$\pm$0.6  & 34.0$\pm$1.4  & 37.8$\pm$1.1  & \textbf{39.3$\pm$1.3}  \\
			Product - Clip       & 32.7$\pm$0.9  & \textbf{37.2$\pm$1.0}  & 33.8$\pm$0.5  & 28.4$\pm$0.7  & 28.4$\pm$0.6  & 29.7$\pm$0.5  & 31.8$\pm$0.8  & 24.9$\pm$1.0  & 30.9$\pm$0.8  & 29.3$\pm$0.9 \\
			Product - Real       & \textbf{62.1$\pm$1.3}  & \textbf{62.5$\pm$1.2}  & \textbf{62.6$\pm$0.7}  & 58.1$\pm$0.5  & 57.6$\pm$0.6  & 59.3$\pm$0.6  & 57.1$\pm$0.8  & 59.2$\pm$0.9  & 60.5$\pm$0.6  & \textbf{62.2$\pm$0.7}  \\
			Real - Product        & 68.3$\pm$1.0  & \textbf{70.4$\pm$0.8}  & \textbf{70.2$\pm$0.5}  & 61.7$\pm$0.8  & 63.4$\pm$0.9  & 61.5$\pm$1.0  & 65.5$\pm$0.6  & 64.5$\pm$1.5  & 64.8$\pm$3.6  & 66.5$\pm$1.1 \\
			Real - Art       & \textbf{40.3$\pm$0.9}  & \textbf{41.3$\pm$1.0}  & 39.2$\pm$0.7  & 33.5$\pm$1.3  & 31.6$\pm$1.5  & 36.9$\pm$0.9  & 36.1$\pm$0.9  & 36.9$\pm$1.9  & 39.9$\pm$1.4  & 39.2$\pm$1.6  \\
			Real - Clip      & \textbf{42.7$\pm$1.1}  & 40.9$\pm$1.0  & 40.4$\pm$0.5  & 35.6$\pm$0.8  & 34.9$\pm$0.9  & 40.4$\pm$0.5  & 35.6$\pm$0.8  & 35.6$\pm$2.0  & 38.7$\pm$2.1  & 38.8$\pm$2.5  \\\hline \hline
			\#Wins (/34)  &   7 & 9 & 5 & 0 & 1 & 0 & 2 & 9 & 12 & 21 \\
			Aver. Rank  & 4.16 & 4.73 & 5.32 & 6.97 & 8.38 & 6.59 & 7.57 & 4.95 & 3.38 & 2.95 \\\hline
	\end{tabular}}
\end{table*}

\section{Conclusion}
\label{sec:conclusion}

The paper proposed a strategy for handling generalized target shift in domain adaptation. It builds upon the simple idea that if the target label proportion where known, then reweighting class-conditional probabilities in the source domain
is sufficient for designing a distribution discrepancy that takes into account 
those shifts. In practice, our algorithm  
estimates the label proportion using Gaussian Mixture models or agglomerative clustering and then matches source and target class-conditional components for allocating the label proportion estimations. Resulting label proportion is then plugged into an  weighted Wasserstein distance. When used for adversarial domain adaptation, we show that our approach outperforms competitors and is able to  adapt to imbalance in target domains.

\rev{
Several points are worth to be extended in future works. 
Our main assumption, for achieving estimations of class-conditionals, is the cyclical monotonicity of the class-conditional distributions in the latent space. However, unfortunately, we do not have any method for checking whether
this assumption holds after training the representation on the source domain, especially as it supposed the knowledge of the class in the target domain.
Hence, it would  be interesting to enforce  this assumption to hold, for instance by defining a regularization term based on the notion of cyclical monotonicity. \\
Furthermore, at the present time,
we have considered simple mean-based approach for matching distributions, it is worth investigating whether higher-order moments are useful for improving the 
matching. Our algorithm relies mostly on our ability to estimate label proportion,
we would be interested on in-depth theoretical analysis label proportion estimation and their convergence and convergence rate guarantees.
}

{
\section*{Acknowledgments}
This work benefited from the support of the project OATMIL ANR-17-CE23-0012 of the French, LEAUDS ANR-18-CE23, was performed using computing resources of CRIANN (Normandy, France), Chaire AI RAIMO and OTTOPIA, 3IA Côte d'Azur Investments ANR-19-P3IA-0002 of the French National Research Agency (ANR). This research was produced within the framework of Energy4Climate Interdisciplinary Center (E4C) of IP Paris and Ecole des Ponts ParisTech. This research was supported by 3rd Programme d’Investissements d’Avenir ANR-18-EUR-0006-02. This action benefited from the support of the Chair "Challenging Technology for Responsible Energy" led by l’X – Ecole polytechnique and the Fondation de l’Ecole polytechnique, sponsored by TOTAL.
}
{}

\bibliographystyle{plain}
\clearpage
\onecolumn

\setcounter{proposition}{0}
\setcounter{lemma}{0}
\setcounter{theorem}{0}
\setcounter{algorithm}{0}
\section*{\centering Supplementary material for \\Match and Reweight for Generalized Target Shift}

This supplementary material presents some details of the theoretical and algorithmic aspects of the work as well as as some additional results. They are listed as below.

\begin{enumerate}
	
	\item Theoretical details and proofs 
		\item Dataset details and architecture details are given in Section \ref{sec:data} and \ref{sec:architect}
	
	\item Figure \ref{fig:exampletoydata} presents some samples of the 3-class toy data set for different configurations of covariance matrices
	making the problem easy, of mid-difficulty or difficult.
		\item Figure \ref{fig:toyshift} exhibits the performances of the compared algorithms depending on the shift of the class-conditional distributions.
	\item Figure \ref{fig:toyimbalancegmm} shows  for the imbalanced toy problem, the results obtained by all competitors including a GMM.
	\item Table \ref{tab:resultsgmm} shows the performance of Source only and a simple GMM+OT on a Visda 3-class problem.
	\item Table \ref{tab:dataconfig} depicts the different configurations of the dataset we used in our experiments
				\end{enumerate}

\section{Theoretical and algorithmic details}
\label{sec:theoretical}

\subsection{Lemma 1 and its proof}

\begin{lemma} 
	For all $p_T^y$, $p_S^y$ and for any continuous class-conditional density distribution $p_S^k$ and $p_T^k$ such that for all $z$ and $k$, we have $p_S(z|y=k)>0$ and $p_S(y=k)>0$.
	the following inequality holds.
	$$
	\sup_{k,z} [w(z) S_k(z)] \geq 1
	$$
	with $S_k(z) = \frac{p_T^g(z|y=k)}{p_S^g(z|y=k)}$ and $w(z) = \frac{p_T^{y=k}}{p_S^{y=k}}$, if $z$ is of class $k$.
\end{lemma} 
\begin{proof}
	Let first show that for any $k$ the ratio $\sup_z \frac{p_T^k}{p_S^k} \geq 1$. 
	Suppose that there does  not exist a $z$  such that $\frac{p_T^k}{p_S^k} \geq 1$. This means that : $\forall z\,\, p_T^k<p_S^k$. By integrating those positive and continuous functions on their domains lead to the contradiction that the integral of one of them is not equal to 1. Hence, there must exists a $z$ such that $\frac{p_T^k}{p_S^k} \geq 1$. Hence, we indeed have ratio $\sup_z \frac{p_T^k}{p_S^k} \geq 1$.
	
	Using a similar reasoning, we can show that $\sup_k \frac{p_T^{y=k}}{p_S^{y=k}} \geq 1 $. For a sake of completeness,
	we provide it here. Assume that $\forall k,\,\,p_T^{y=k}<p_S^{y=k}$.
	We thus have $\sum_k p_T^{y=k}< \sum_k p_S^{y=k}$. Since noth sums should be equal to $1$ leads to a contradiction.
	
	By exploiting these two inequalities, we have :
	$$\sup_{k,z} [w(z) S_k(z)] = \sup_{k} \left [w(z) \sup_z S_k(z)\right]\geq \sup_{k} w(z)  \geq 1$$
	
\end{proof}

\subsection{Theorem 1 and its proof}

\begin{theorem}
	Under the assumption of Lemma \ref{lemme:bound}, and assuming that any function $h \in \mathcal{H}$ is $K$-Lipschitz and $g$ is a continuous function then
	for every function  $h$ and $g$, we have
	$$\varepsilon_{T}(h \circ g,f) \leq \varepsilon_{S}(h \circ g,f)  + 2K \cdot WD_1(p_S^g,p_T^g)
	+ \left[1 + \sup_{k,z} w(z)S_k(z))\right] \varepsilon_{S}(h^\star \circ g,f) + \varepsilon_{T}^z(f_S^g,f_T^g)$$
	where $S_k(z)$ and $w(z)$ are as defined in Lemma \ref{lemme:bound}, $h^\star = \argmin_{h \in \mathcal{H}} \varepsilon_{S}(h \circ g;f)$ 
	and $\varepsilon_{T}^z(f_S^g,f_T^g) = \mathbb{E}_{z \sim p_T^z}[|f_T^g(z) - f_S^g(z)|] $
\end{theorem}
\begin{proof}
	At first, let us remind the following result due to \citet{shen2018wasserstein}.
	Given two probability distributions $p_S^g$ and $p_T^g$, we have
	$$
	\varepsilon_S^z(h,h^\prime) - \varepsilon_T^z(h,h^\prime) \leq 2K \cdot WD_1(p_S^g,p_T^g)
	$$
	for every hypothesis $h$,\,$h^\prime$ in $\mathcal{H}$. Then, we have the following bound for the target error
	\begin{align}
	\varepsilon_{T}(h \circ g,f) &\leq \varepsilon_{T}(h \circ g,h^\star \circ g) + 
	\varepsilon_{T}(h^\star \circ g,f) \label{eq:1}\\
	&\leq \varepsilon_{T}(h \circ g,h^\star \circ g) + \varepsilon_{S}(h \circ g,h^\star \circ g) - \varepsilon_{S}(h \circ g,h^\star \circ g) +
	\varepsilon_{T}(h^\star \circ g,f) \label{eq:2}\\
	& \leq \varepsilon_{S}(h \circ g,h^\star \circ g) + 	\varepsilon_{T}(h^\star \circ g,f)  + 2K\cdot WD_1(p_S^g,p_T^g) \label{eq:3}\\
	& = \varepsilon_{S}^z(h,h^\star) + 	\varepsilon_{T}^z(h^\star,f_T^g)  + 2K\cdot WD_1(p_S^g,p_T^g) \label{eq:4} \\
	&  \leq \varepsilon_{S}^z(h,f_S^g) + \varepsilon_{S}^z(h^\star,f_S^g) +	\varepsilon_{T}^z(h^\star,f_T^g)  + 2K\cdot WD_1(p_S^g,p_T^g) \label{eq:5} \\
	&  \leq \varepsilon_{S}^z(h,f_S^g) + \varepsilon_{S}^z(h^\star,f_S^g) +	\varepsilon_{T}^z(h^\star,f_S^g)+ \varepsilon_{T}^z(f_S^g,f_T^g) + 2K\cdot WD_1(p_S^g,p_T^g)\label{eq:6}
	\end{align}
	where the lines \eqref{eq:1}, \eqref{eq:5}, \eqref{eq:6} have been obtained using triangle inequality, Line \eqref{eq:3} by using $\varepsilon_{U}(h \circ g,h^\star \circ g) =  \varepsilon_{U}^z(h,h^\star)$ and by applying Shen's et al.
	above inequality, Line \eqref{eq:4} by using $\varepsilon_{U}(h \circ g,f) =  \varepsilon_{U}^z(h,f_U^g)$.
	Now, let us analyze the term $\varepsilon_{S}^z(h^\star,f_S^g) +	\varepsilon_{T}^z(h^\star,f_S^g)$. Denote as $r_S(z) = |h^\star(z)-f_S^g(z)|$. Hence, we have
	\begin{align}
	\varepsilon_{S}^z(h^\star,f_S^g) +	\varepsilon_{T}^z(h^\star,f_S^g) &= \int r_S(z) [p_S^g(z) + p_T^g(z)] dz \\
	&= \sum_k p_S(y=k) \int r_S(z)p_S^g(z|y=k)\big[1 + \frac{p_T(y=k)}{p_S(y=k)}S_k(z)\big] dz \label{eq:10}\\
	& \leq \left(1 + \sup_{k,z} [w(z) S_k(z)] \right) 	\varepsilon_{S}^z(h^\star,f_S^g)
	\end{align}
	where Line \eqref{eq:10} has been obtained by expanding marginal distributions. Merging the last inequality into Equation \eqref{eq:6} concludes the proof.
\end{proof}

\subsection{Proposition 1 and its proof }

\begin{proposition}
	Denote as $ \nu = \frac{1}{C}\sum_{j=1}^{C}\delta_{p_S^j}$ and
	$ \mu = \frac{1}{C}\sum_{j=1}^{C}\delta_{p_T^j}$, representing
	respectively the class-conditional probabilities in source and target domain. 
	Given $\mathcal{D}$ a  distance over probability distributions, 
	Assume that for any permutation $\sigma$ of $C$ elements, the following assumption,
	known as the $\mathcal{D}$-cyclical monotonicity relation,
	holds	$$\sum_j \mathcal{D}(p_S^j,p_T^j)\leq \sum_j \mathcal{D}(p_S^j,p_T^{\sigma(j)})$$
		then solving the optimal transport problem between
	$\nu$ and $\mu$ as defined in equation \eqref{eq:wd}  using  $\mathcal{D}$ as the ground cost matches correctly class-conditional probabilities.
\end{proposition}
\begin{proof}
	The solution $\P^*$ of the OT problem lies on an extremal point of $\Pi_C$. Birkhoff's theorem
	\cite{birkhoff:1946} states that the set of extremal points of $\Pi_C$ is the set of permutation 
	matrices so that there exists an optimal solution of the form $\sigma^* :  [1,\cdots,C] \rightarrow [1,\cdots,C]$. 
	The support of $\P^*$ is $\mathcal{D}$-cyclically monotone~\citep{ambrosio2013user,santambrogio2015optimal} (Theorem 1.38), meaning that 
	$ \sum_j^C \mathcal{D} (p_S^j,p_T^{\sigma^*(j)}) \leq  \sum_j^C \mathcal{D} (p_S^j,p_T^{\sigma(j)}), \forall \sigma \neq \sigma^*.$
						Then, by hypothesis, $\sigma^*$ can be identified to the identity permutation, and solving the optimal assignment problem 
	matches correctly class-conditional probabilities.
	\end{proof}

\subsection{Proposition 2 and its proof}

\begin{proposition}
	Denote as $\gamma$ 	the optimal coupling plan for distributions $\nu$ and $\mu$  with balanced class-conditionals such that  $ \nu = \frac{1}{C}\sum_{j=1}^{C}\delta_{p_S^j}$ and
	$ \mu = \frac{1}{C}\sum_{j=1}^{C}\delta_{p_T^j}$
	under assumptions given in  Proposition \ref{prop:dist}. 	Assume that the classes are ordered so that we have $\gamma= \frac{1}{C}
	\text{diag}(\1)$ then  $\gamma'=\text{diag}(\a)$ is also optimal for the transportation problem
	with  marginals $ \nu^\prime = \sum_{j=1}^{C} a_j \delta_{p_S^j}$ and
	$ \mu^\prime = \sum_{j=1}^{C} a_j \delta_{p_T^j}$, with $a_j > 0, \forall j$. In addition,
	if the Wasserstein distance between $\nu^\prime$ and $\mu^\prime$ is $0$, it implies
	that the distance between class-conditionals are all $0$.
\end{proposition}
\begin{proof}
		By assumption and without loss of generality, the
	class-conditionals are arranged so that  $\gamma = \frac{1}{C}
	\text{diag}(\1)$.
	Because the weights in the marginals are
	not uniform anymore, $\gamma$  is not a feasible solution for the OT problem with 
	$\nu^\prime$ and  $\mu^\prime$
	but $\gamma^\prime = \text{diag}(\a)$ is. 
	Let us now show that any feasible non-diagonal plan $\Gamma$  has higher cost than
	$\gamma^\prime$ and thus is not optimal. At first, consider any permutation $\sigma$ of $C$ elements and its corresponding permutation matrix $\P_\sigma$, because $\gamma = \frac{1}{C}
	\text{diag}(\1)$ is optimal, the cyclical monotonicity relation $\sum_i \mathcal{D}_{i,i}\leq \sum_i \mathcal{D}_{i,\sigma(i)}$  holds true $\forall \sigma$. 
	Starting from $\gamma^\prime = \text{diag}(\a)$, any direction $\Delta_\sigma= -\I+\P_\sigma$ is a feasible direction (it does not violate the marginal constraints) and due to the  cyclical monotonicity, any move in this direction will increase the cost. Since any non-diagonal $\gamma_ z\in\Pi(\a,\a)$ can be reached with a sum of displacements $\Delta_\sigma$ (property of the Birkhoff polytope) it means that the transport cost induced by  $\gamma_z$ will always be greater or equal to the cost for the diagonal $\gamma^\prime$ implying that $\gamma^\prime$ is the solution of the OT problem with marginals $\a$.\\
	As a corollary, it is straightforward to show that 
	$
	W(\nu^\prime, \mu^\prime) = \sum_{i=1}^C \mathcal{D}_{i,i} a_i = 0 \implies \mathcal{D}_{i,i} = 0
	$
	as $a_i >0$ by hypothesis.
\end{proof}

\section{Experimental Results}
\label{sec:experimental}
\subsection{Dataset details}
\label{sec:data}
We have considered $4$ family of domain adaptation problems based on 
the digits, Visda, Office-31 and Office-Home dataset. For all these datasets,
we have not considered the natural train/test number of examples, in order to 
be able to build different label distributions at constant number of examples (suppose one class has at most 800 examples, if we want that class to represent $80\%$ of the samples, then we are limited to
 $1000$ samples).

For the digits problem, We have used the MNIST, USPS and the MNITSM datasets.
we have learned the feature extractor from scratch and considered the following
train-test number of examples setting. For MNIST-USPS, USPS-MNIST and MNIST-MNISTM, we have respectively used 60000-3000, 7291-10000, 10000-10000.

The VisDA 2017 problem is a $12$-class classification problem with source and target domain being simulated and real images. We have considerd two sets of problem, a 3-class one (based on the classes \emph{aeroplane, horse} and \emph{truck}) and the full 12-class problem.

The Office-31 is an object categorization problem involving $31$ classes with a total of 4652 samples. There exists $3$ domains in the problem based on
the source of the images : Amazon (A), DSLR (D) and WebCam (W). We have considered all possible pairwise source-target domains. 

The Office-Home is another object categorization problem involving $65$ classes with a total of 15500 samples. There exists $4$ domains in the problem based on the source of the images : Art, Product, Clipart (Clip), Realworld (Real).

For the Visda and Office datasets, we have considered Imagenet pre-trained ResNet-50 features and our feature extractor (which is a fully-connected feedforword networks) aims at adapting those features. We have used pre-trained features
freely available at \url{https://github.com/jindongwang/transferlearning/blob/master/data/dataset.md}.

\subsection{Architecture details}
\label{sec:architect}
\paragraph{Toy} The feature extractor is a 2 layer fully connected network with $200$ units and ReLU activation function. The classifier is also a $2$ layer fully connected network with same number of units and activation function. Discriminators have 3 layers with same number of units. 

\paragraph{Digits} For the MNIST-USPS problem, the architecture of our feature extractor is composed of the two CNN layers with 32 and 20 filters of size $5\times5$ and 2-layer fully connected networks as discriminators with $100$ and $10$ units. The feature extractor uses a ReLU activation function
and a max pooling. For he MNIST-MNISTM adaptation problem we have used the same feature extractor network and discriminators as in \cite{pmlr-v37-ganin15}.

 \paragraph{VisDA} For the VisDA dataset, we have considered pre-trained 2048 features obtained from a ResNet-50 followed by $2$ fully connected networks with $100$ units and ReLU activations. The latent space is thus of dimension $100$. Discriminators and classifiers are also a $2$ layer Fully connected networks with $100$ and respectively 1 and "number of class" units.
 
  \paragraph{Office} For the office datasets, we have considered pre-trained 2048 features obtained from a ResNet-50 followed by two fully connected networks with output of $100$ and $50$ units and ReLU activations. The latent space is thus of dimension $50$. Discriminators and classifiers are also a $2$ layer fully connected networks with $50$ and respectively 1 and "number of class" units.
 
  For Digits and VisDA and Office applications, all models have been trained using ADAM for $100$ iterations with validated learning rate, while for the toy problem, we have used a SGD.

\subsection{Other things we have tried}
\rev{
\begin{itemize}
	\item For estimating the mixture proportion of each component $\{p_T^i\}$, we have proposed
	two clustering algorithms, one based on Gaussian mixture model and another based on agglomerative clustering. We have also tried K-means algorithm but finally opted for the agglomerative clustering as it does not
	need specific initializations (and thus is robust to it). Our early experiments also showed that it provided slightly better performances than K-means. \\
\end{itemize}
}
\begin{figure*}[t]
	\begin{center}
			~\hfill
		\includegraphics[width=4.cm]{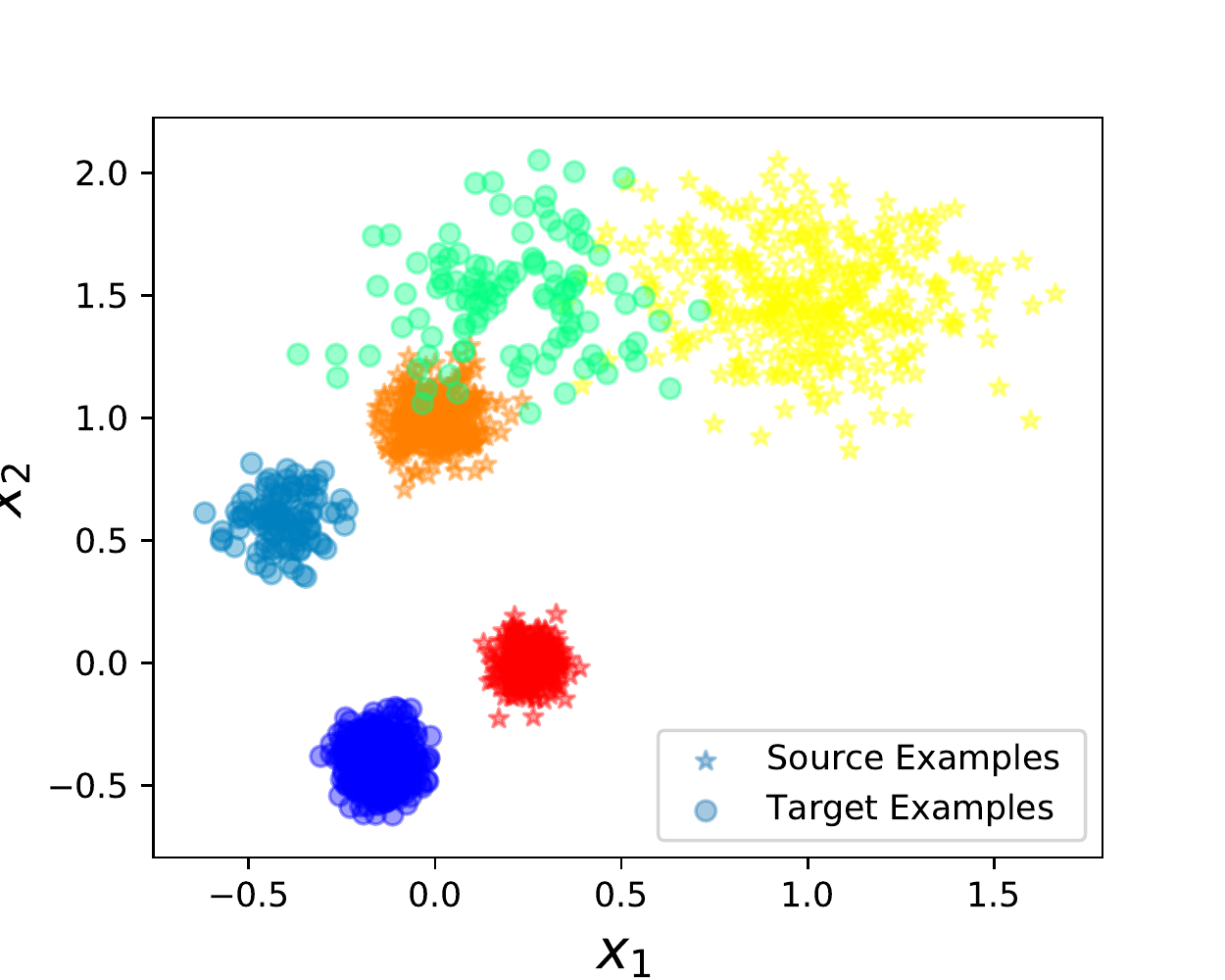}~\hfill~
		\includegraphics[width=4.cm]{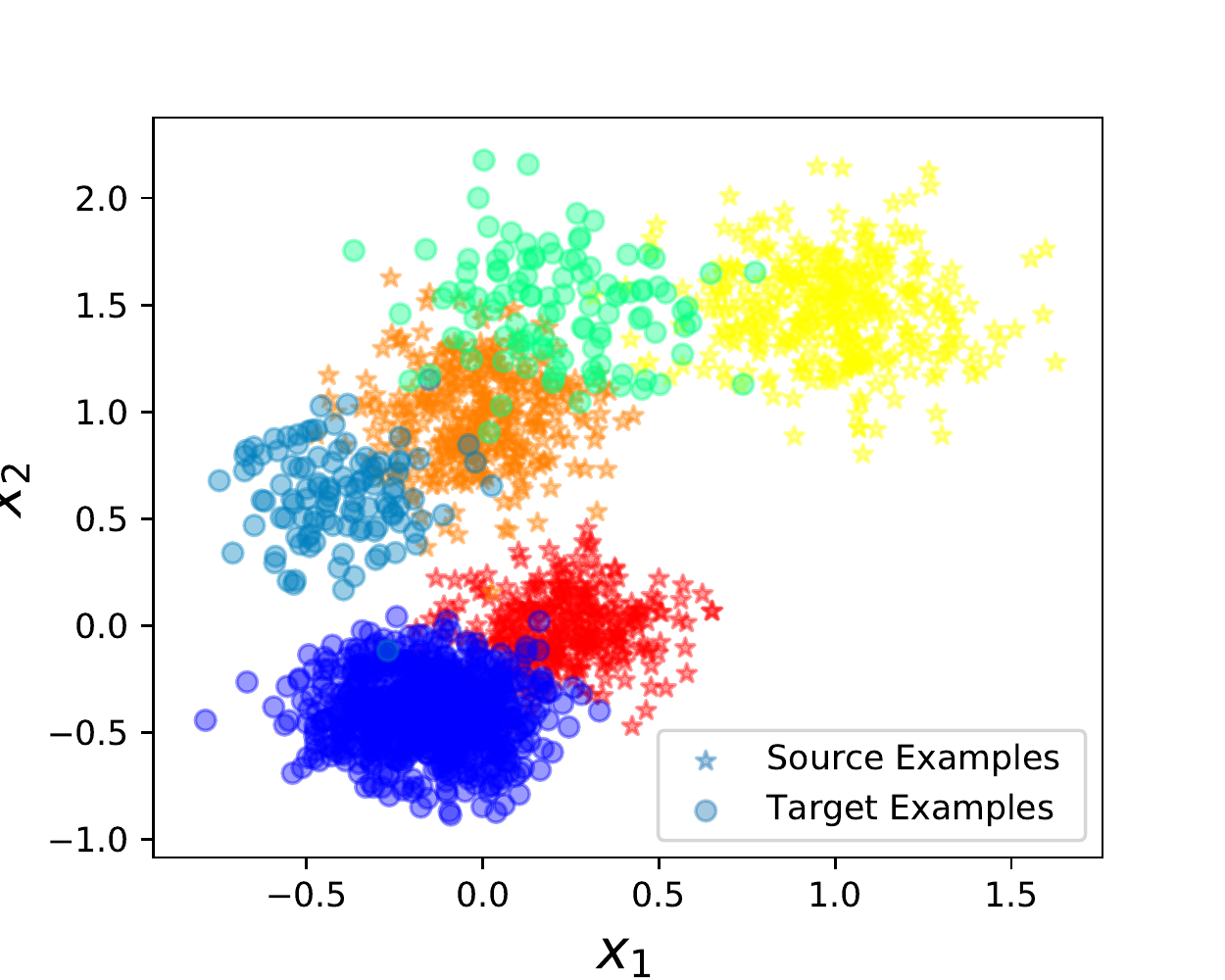}~\hfill~
		~\hfill~\includegraphics[width=4.cm]{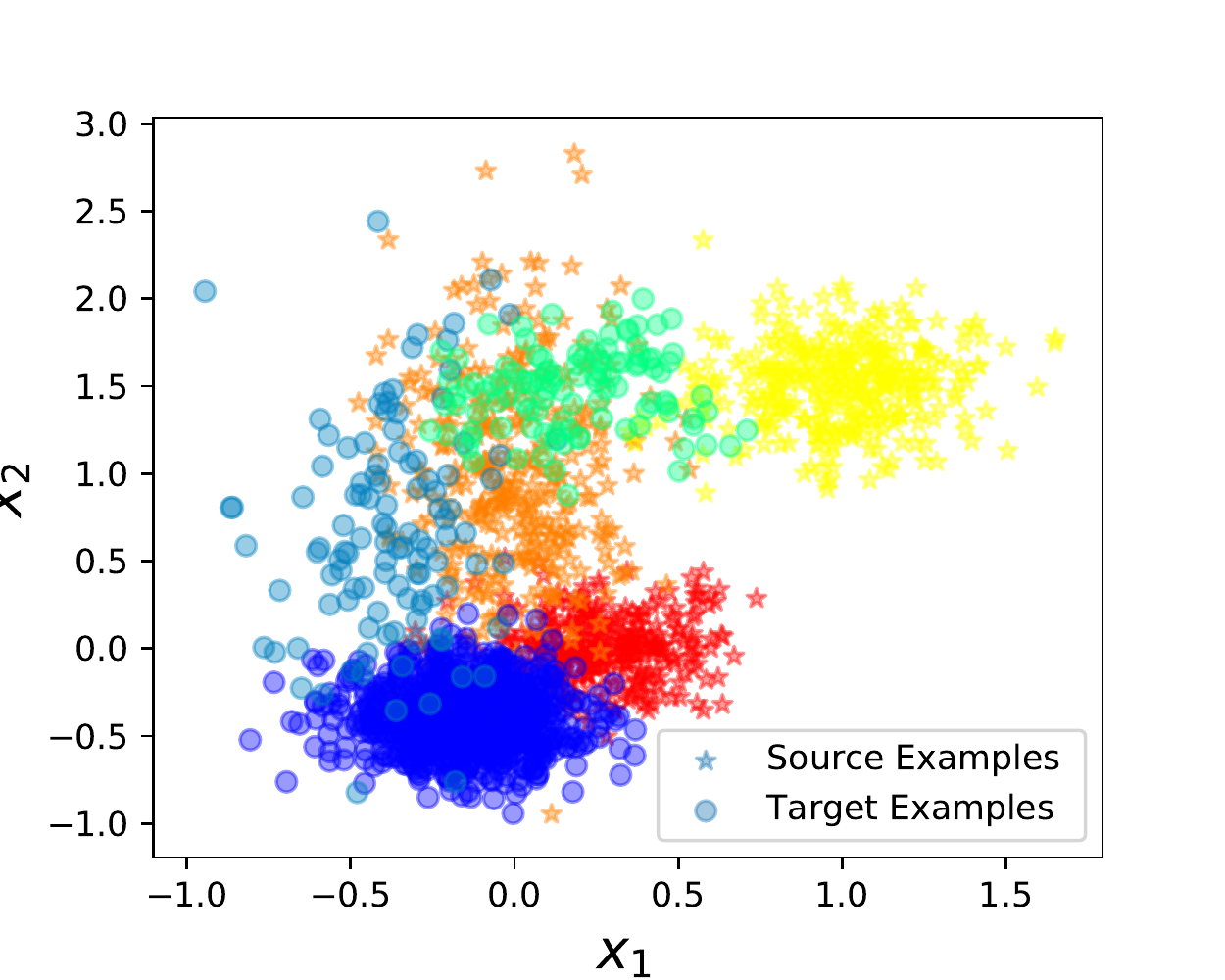}~\hfill~
			\caption{Examples of source and target domain examples. For each domain, data are composed of three Gaussians defining each class. In the source domain, classes are balanced whereas in the target domain, we have a ratio of $0.8,0.1, 0.1$. The three
			configurations presented here vary in their covariance matrices. From left to right,
			we have Gaussians that are larger and larger making them difficult to classify.
		In the most right examples, the second class of the source domain and the third one of the target domain are mixed. This region becomes indecidable for our model as the 
		source loss want to classify it as "Class 2" while the Wasserstein distance want to match it with "Class 3" of the source domain.\label{fig:exampletoydata}}
	\end{center}
\end{figure*}

\begin{figure*}[t]
	\begin{center}
		~\hfill
		\includegraphics[width=4.5cm]{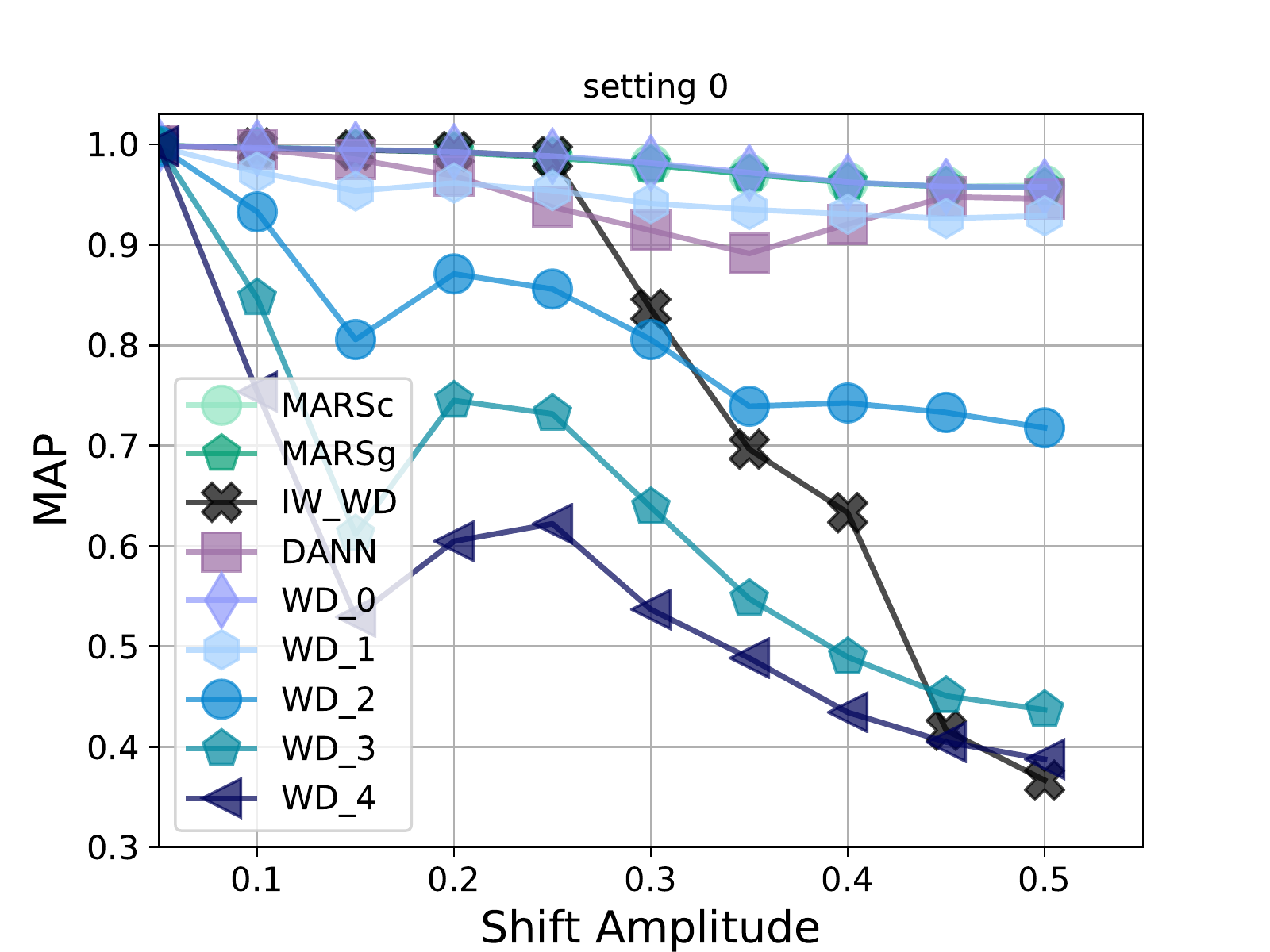}~\hfill~
		\includegraphics[width=4.5cm]{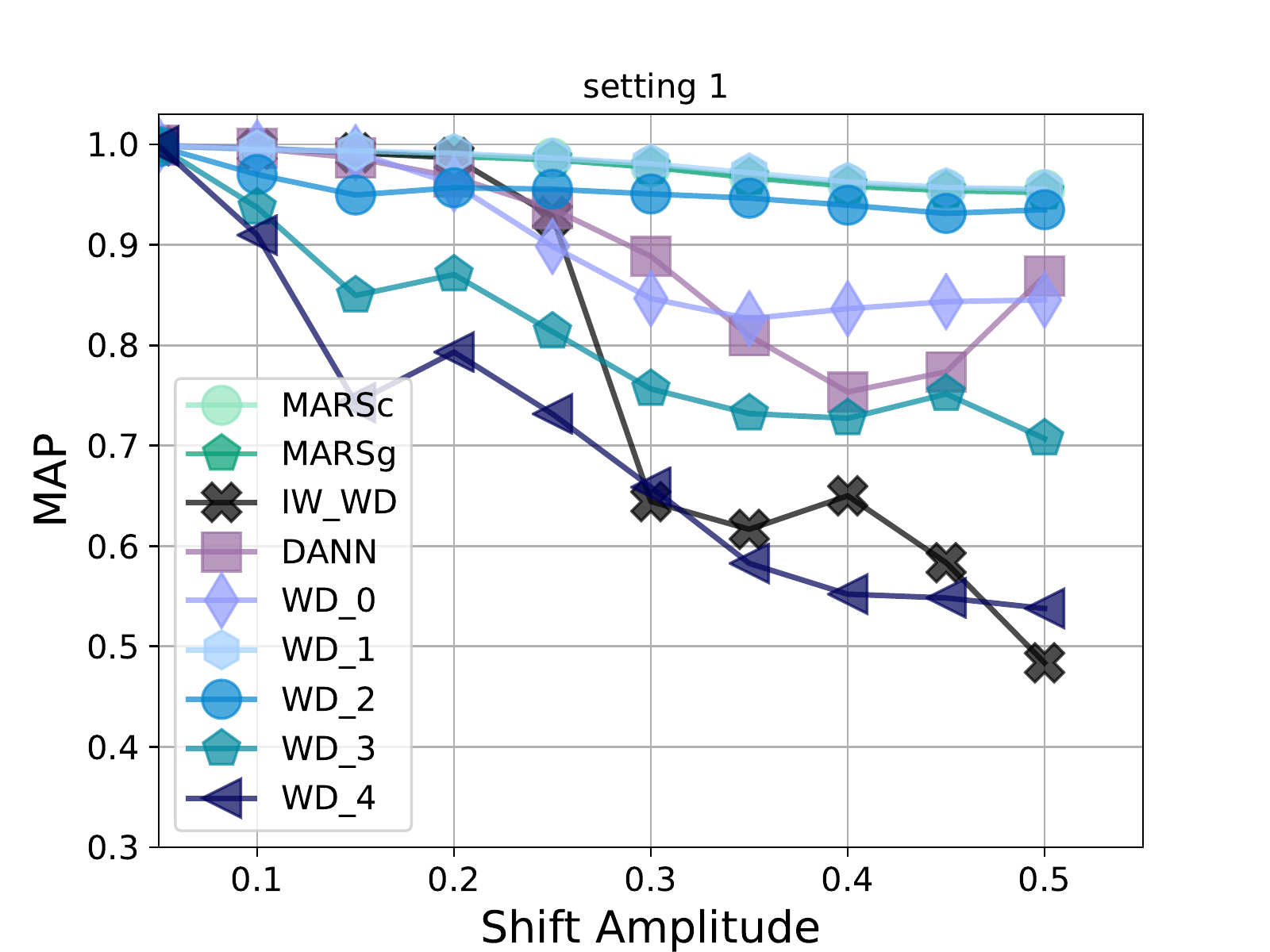}~
		~\hfill~\includegraphics[width=4.5cm]{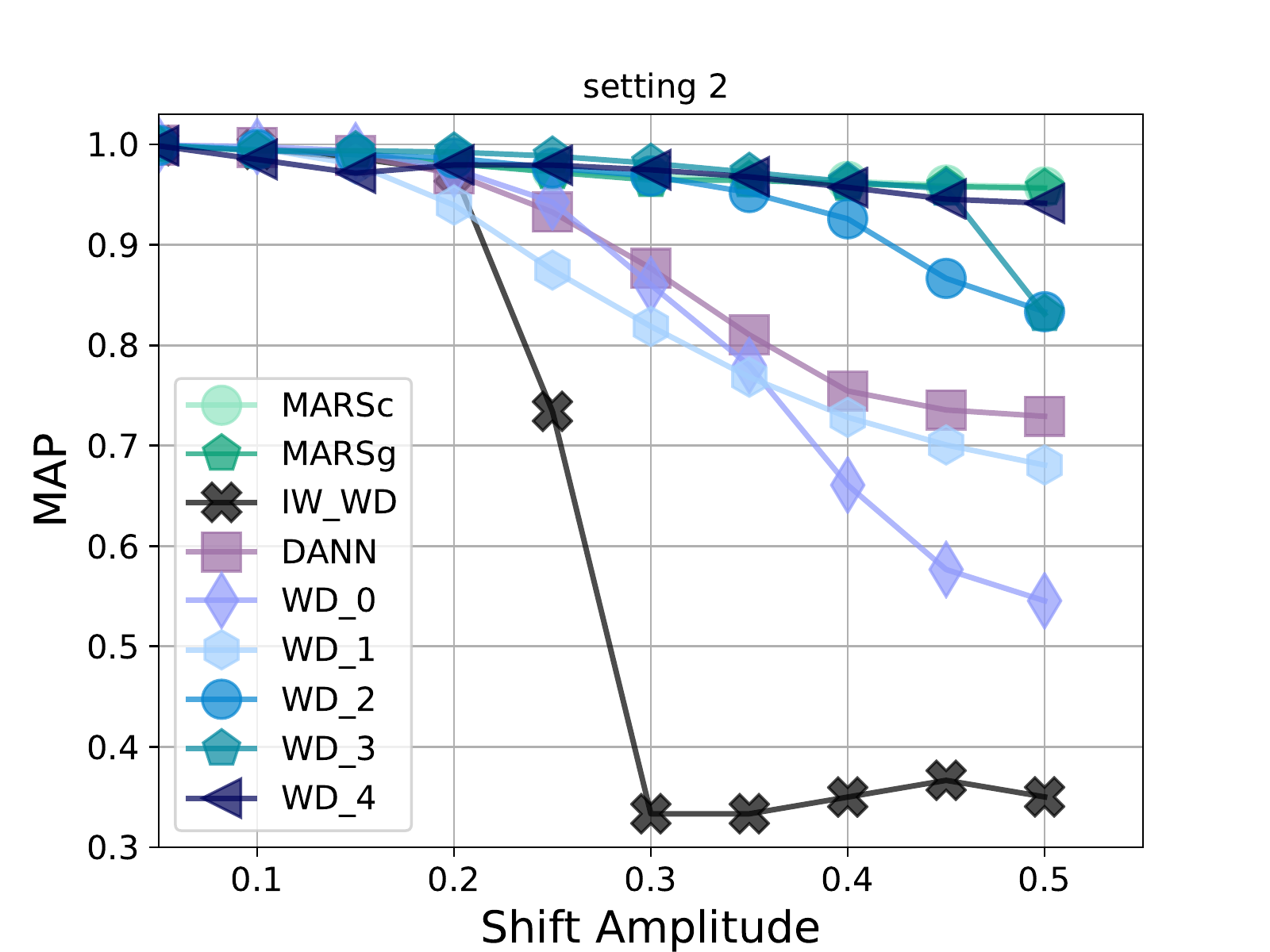}~\hfill~
		
		\caption{Performance of the compared algorithms in different label shift  setting and for increasing shift between means of class-conditionals.  In source domain, label distributions are uniform and shift occurs due to change only in the target domain.  (left) $p_T(y=1) = 0.33$, $p_T(y=2)=0.33$, $p_T(y=3)=0.34$.
			(middle) $p_T(y=1) = 0.5$, $p_T(y=2)=0.2$, $p_T(y=3)=0.2$, (right) $p_T(y=1) = 0.8$, $p_T(y=2)=0.1$, $p_T(y=3)=0.1$.
			 For balanced problems, we note
			that best methods are $\text{WD}_{\beta=\{0,1\}}$, DANN and our approaches 			either using GMM or clustering for estimating label proportion. As expected,
			a too heavy reweighting yields to poor performance for $\text{WD}_{\beta=\{2,3,4\}}$. 			Then for a mild imbalance, $\text{WD}_{\beta=\{1,2\}}$ performs better than the other competitors
			while for higher imbalance,  $\text{WD}_{\beta=\{3,4\}}$ works better. 
			For all settings, our methods are competitive as they are adaptive to the imbalance through the estimation fo $p_T(y)$. The IW-WD of \citet{combes2020domain} performs well until the distance between class-conditionals is too large. This is justified by theory as their estimator of the ratio $p_T(y)/p_S(y)$ is tailored for situations where class-conditionals are equal.
			 \label{fig:toyshift}}
	\end{center}
\end{figure*}
\begin{figure*}[t]
	\begin{center}
		~\hfill
		\includegraphics[width=4.5cm]{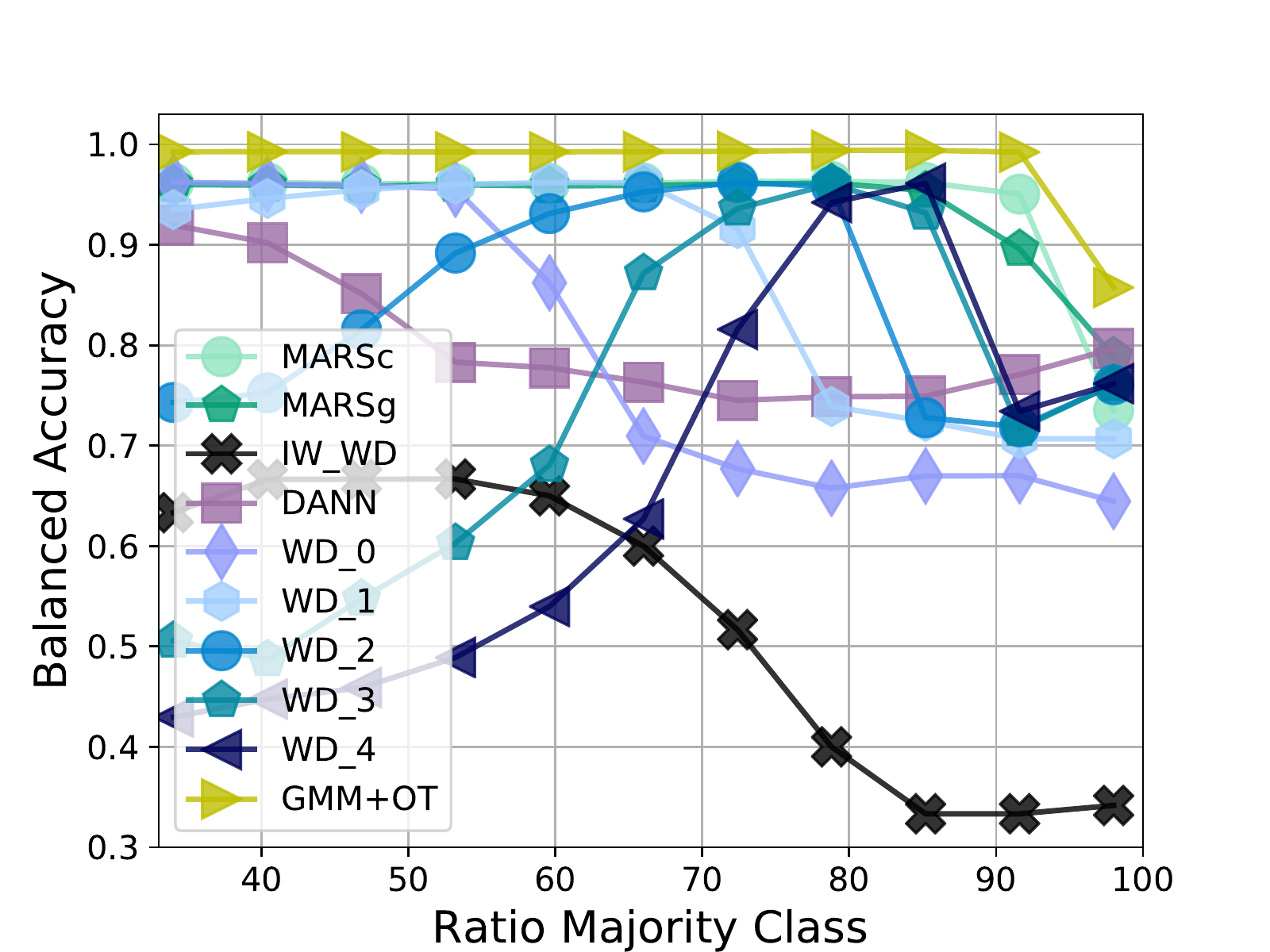}~\hfill~
		\includegraphics[width=4.5cm]{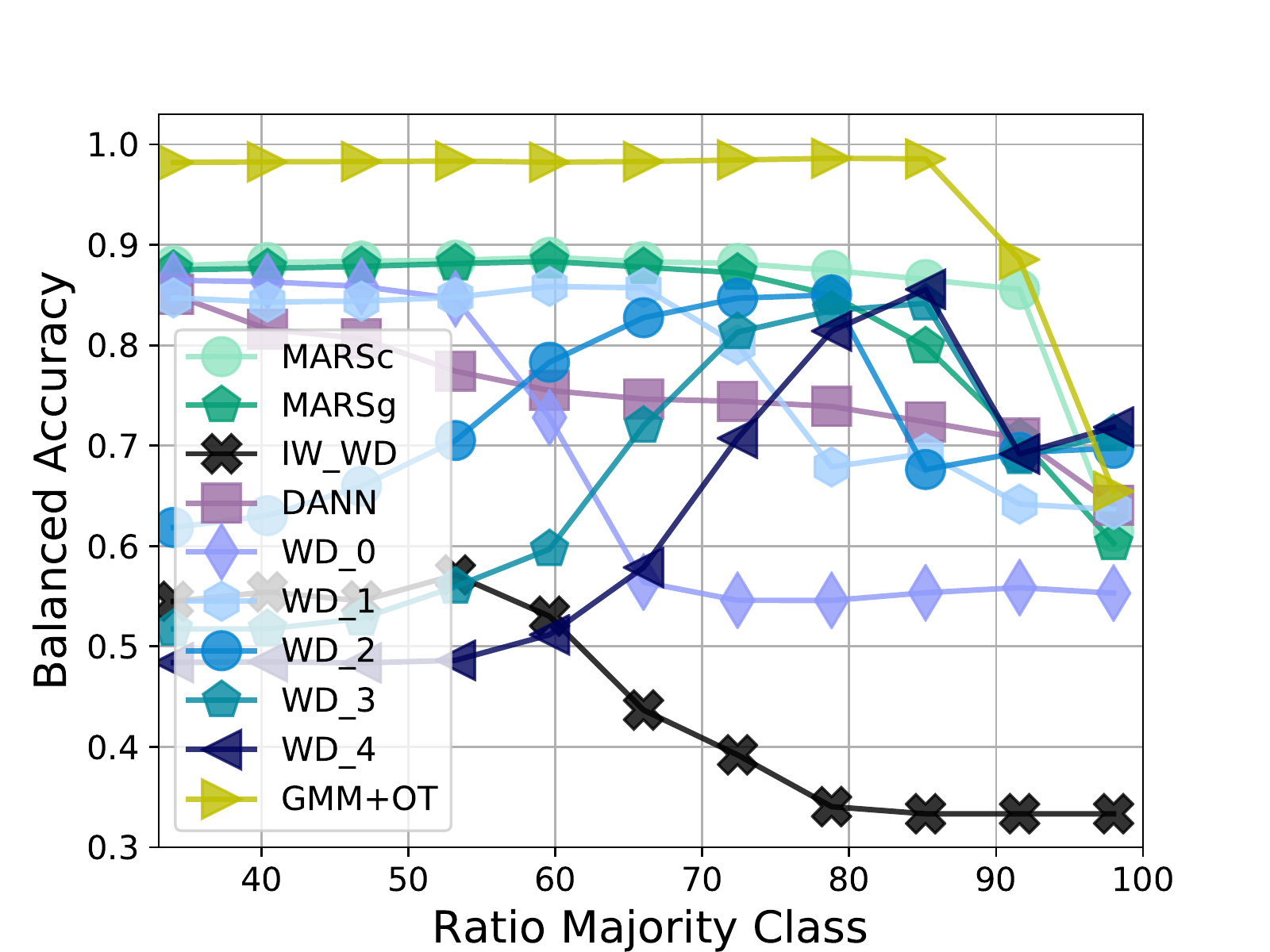}~\hfill~ 
		~\hfill~\includegraphics[width=4.5cm]{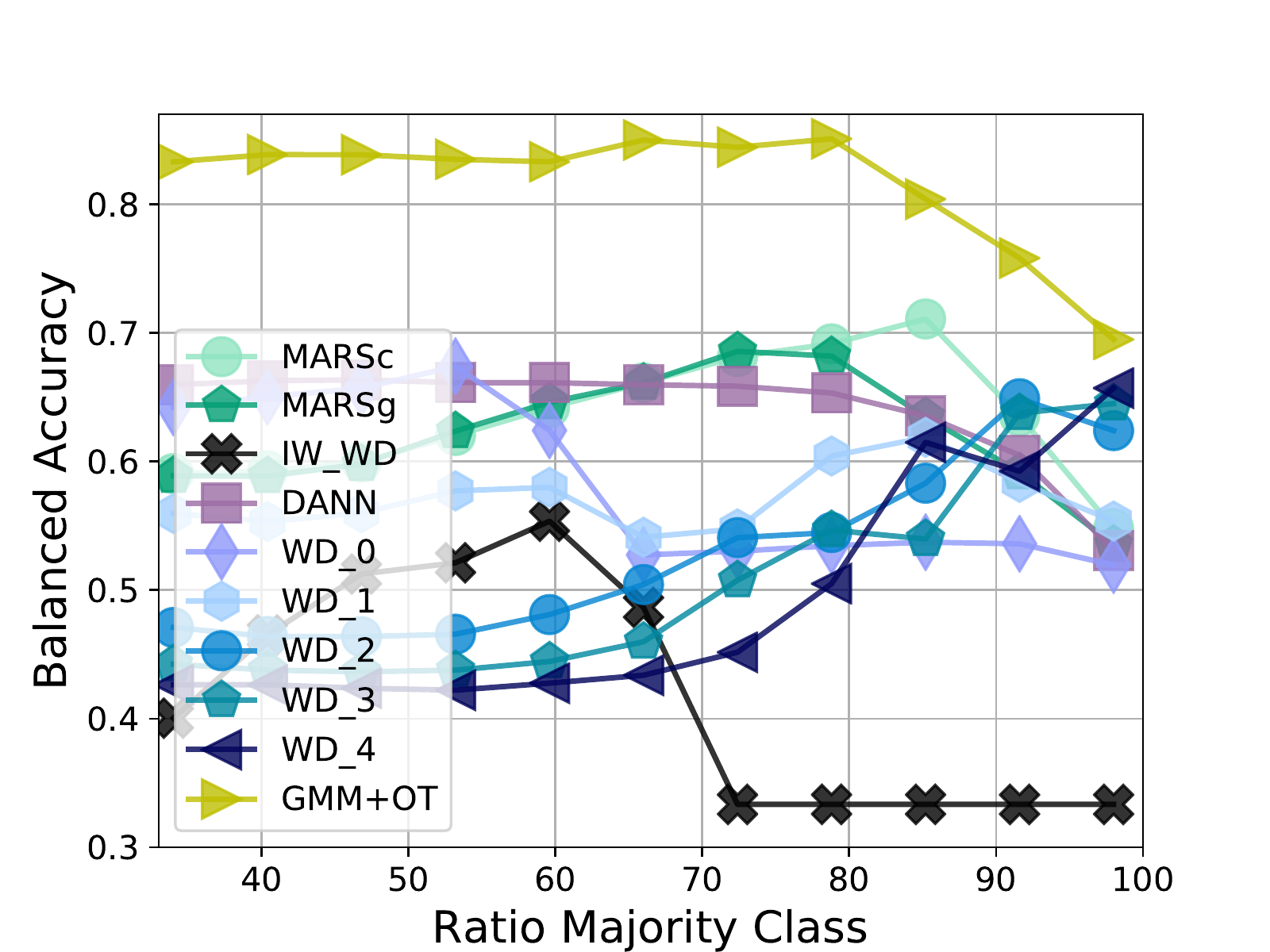}
		\hfill~
		\caption{Performance of the compared algorithms, including \textbf{GMM+OT} for three different
			covariance matrices of the Gaussians composing the toy dataset with respect to the imbalance. The shift between the class-conditionals has been fixed and yields to samples  similar to those presented in Figure \ref{fig:exampletoydata}. Our method is referred as \textbf{MARS}. The x-axis is given with respect to the ratio of majority class which is
			the class $1$.	(left) Low-error setting. (middle) mid-error setting.
			(right) high-error setting. material.
			We note that this toy problem can be easily solved using a GMM and a optimal transport-based label assignment. We can also remark that again as soon as the class-conditionals do not match anymore, the IW-WD of \citet{combes2020domain} fails due to its inability to estimate correctly the importance weight $w$.
			\label{fig:toyimbalancegmm}.}
	\end{center}
\end{figure*}

\begin{table}
	\caption{Comparing Source-Only model and GMM+OT approach on the
		VisDA-3-mode problems. We can note that for these problems where the latent space is of dimension $100$, the GMM+OT compares poorly
		to Source-Only. In addition, we can note that there is very high variability in the performance. \label{tab:resultsgmm}}
	\begin{center}
		\begin{tabular}{lcc} 
			\hline
			Configuration  & Source   & GMM+OT  \\\hline 
			Setting 1 &	 79.3$\pm$4.3		& 81.2$\pm$4.7 \\
			Setting 4 	 &	80.2$\pm$5.3		& 76.3$\pm$9.8 \\
			Setting 2	 &	81.5$\pm$3.5		& 74.8$\pm$10.4 \\
			Setting	3 &	78.4$\pm$3.2		& 70.0$\pm$10.8 \\
			Setting  5 &	83.5$\pm$ 3.5		& 77.0$\pm$10.4  \\
			Setting  6 &	80.8$\pm$4.2		& 72.9$\pm$10.2 \\
			Setting  7 &	79.2$\pm$3.7		& 69.5$\pm$9.8 \\
			
			\hline
		\end{tabular}
	\end{center}
\end{table}

\begin{table*}
		\small
	\caption{Table of the dataset experimental settings. We have considered different domain adaptation problems and different configurations of the label shift in the source and target domain. For the digits and VisDA problem, we provide the ratio of samples of classes for each problem (\eg for the third setting of VisDA-3 problem, the second class accounts for the 70\% of the samples in target domain). For Office datasets, because of large amount of classes, we have changed percent of samples of that class in the source or target (\eg the 10-class in Office Home uses respectively 20\% and 100\% of its sample  for the source and target domain).   		
		 \label{tab:dataconfig}}
	\vspace{0.5cm}

	\begin{tabular}{lcc} 
		\hline
		Configuration     & Proportion Source & Proportion Target \\\hline 
		MNIST-USPS balanced    & \{$\frac{1}{10},\cdots, \frac{1}{10}$\} & \{$\frac{1}{10},\cdots, \frac{1}{10}$\}  \\
		MNIST-USPS mid   & \{$\frac{1}{10},\cdots, \frac{1}{10}$\} & $\{0,\cdots,3,6\}=0.02, \{4,5\}=0.02, \{7,8,9\}=0.1$ \\
		MNIST-USPS high    & \{$\frac{1}{10},\cdots, \frac{1}{10}$\} & $\{0\}=0.3665,\{1\}=0.3651 ,\{2,\cdots\}=0.0335$  \\
		\hline
		USPS-MNIST balanced    & \{$\frac{1}{10},\cdots, \frac{1}{10}$\} & \{$\frac{1}{10},\cdots, \frac{1}{10}$\}  \\
		USPS-MNIST mid    & \{$\frac{1}{10},\cdots, \frac{1}{10}$\} & $\{0,\cdots,3,6\}=0.02, \{4,5\}=0.02, \{7,8,9\}=0.1$ \\
		USPS-MNIST high   & \{$\frac{1}{10},\cdots, \frac{1}{10}$\} & $\{0\}=0.3665,\{1\}=0.3651,\{2,\cdots\}=0.0335$  \\
		\hline
		
		MNIST-MNISTM (1)   & $\{0-4\}=0.05, \{5-9\}=0.15$& $\{0,\cdots,3,6\}=0.02, \{4,5\}=0.02, \{7,8,9\}=0.1$ \\
		MNIST-MNISTM (2) 	  & $\{0-2\}=0.26, \{3-9\}=0.03$& $\{0-6\}=0.03, \{7-9\}=0.26$  \\
		MNIST-MNISTM (3)   & $\{0-5\}=0.05, \{6-9\}=0.175$& $\{0-3\}=0.175, \{4-9\}=0.05$  \\\hline
		VisDA-3 (1)  & \{0.33,0.33,0.34\} & \{0.33,0.33,0.34\} \\
		VisDA-3 (2) 		 & \{0.4,0.2,0.4\} & \{0.2,0.6,0.2\} \\
		VisDA-3 (3) 		 & \{0.4,0.2,0.4\} & \{0.15,0.7,0.15\} \\
		VisDA-3 (4)		 & \{0.4,0.2,0.4\} & \{0.1,0.8,0.1\} \\
		VisDA-3 (5)		 & \{0.6,0.2,0.2\} & \{0.2,0.2,0.6\}\\ 
		VisDA-3  (6)  & \{0.6,0.2,0.2\} & \{0.15,0.2,0.65\} \\
		VisDA-3 (7)	 & \{0.6,0.2,0.2\} & \{0.2,0.65,0.15\} 
		
		\\\hline
		VisDA-12 (1)   & \{$\frac{1}{12},\cdots, \frac{1}{12}$\} & \{$\frac{1}{12},\cdots, \frac{1}{12}$\}  \\
		VisDA-12 (2) 	 & \{$\frac{1}{12},\cdots, \frac{1}{12}$\} & $\{0-3\}=0.15, \{4-11\}=0.05$  \\
		VisDA-12 (3) 	 & \{$\frac{1}{12},\cdots, \frac{1}{12}$\} & $\{0-1\}=0.2,\{2-5\}=0.1, \{6-11\}=0.03$  \\\hline
		Office-31  & $\{0-15\}: 30\%$ $\{15-31\}: 80\%$ &  $\{0-15\}: 80\%$ $\{15-31\}: 30\%$ \\
		Office-Home  & $\{0-32\}: 20\%$ $\{33-65\}: 100\%$ &  $\{0-32\}: 100\%$ $\{33-65\}: 20\%$ \\\hline				
	\end{tabular}
\end{table*}

\end{document}